\documentclass[authoryear,preprint,12pt]{elsarticle}
\usepackage[spanish,activeacute]{babel}
\usepackage[latin1]{inputenc}
\usepackage {amsmath, amsthm, amssymb}[]
 \usepackage{latexsym}
 \usepackage[dvips,final]{epsfig}

\newtheorem{postulado}{Postulado }
\newtheorem{theorem}{Teorema}
\newtheorem{lemma}{Lema}
\newtheorem{corolario}{Corolario}
\newtheorem* {defNoNum}{Definición}
\newtheorem{definicion} {Definición}
\newtheorem{ejemplo}{Ejemplo}

\newcommand{\funi}{\footnotesize\textbf{i\hspace{-0.26
cm}\normalsize$\bigcirc$
}}

\begin{document}

\newpage
\setcounter{page}{1}

\begin{frontmatter}

\title{Hacia una teoría de unificación para los comportamientos cognitivos}

\author{Sergio Miguel Tomé \fnref{fn1}  }

\fntext[fn1]{ Cualquier comentario sobre este trabajo puede ser enviado a smt@smt.name }

\address{}

\begin{abstract}
Cada ciencia cognitiva intenta comprender un conjunto de comportamientos cognitivos concreto. La estructuración del conocimiento sobre este aspecto de la naturaleza, dista mucho de lo que se puede esperar de una ciencia. No se ha hallado una manera global de explicar consistentemente el  conjunto de todos los comportamientos cognitivos y sobre muchas cuestiones simplemente hay opiniones de miembros de la comunidad científica. Este artículo presenta tres propuestas. La primera es la de proponer a la comunidad científica la necesidad de plantearse seriamente el problema de la unificación de los comportamientos cognitivos. La segunda propuesta es reivindicar en el estudio científico de los comportamientos cognitivos la aplicación de las reglas para razonar sobre la naturaleza que aportó Newton en su libro Philosophiæ Naturalis Principia Mathematica. La tercera propuesta del artículo es una teoría científica en desarrollo que sigue las reglas planteadas por Newton para razonar sobre la naturaleza y que podría llegar a explicar de una manera global todos los comportamientos cognitivos.

\end{abstract}

\begin{keyword}
Comportamientos cognitivos\sep teoría de unificación.
\end{keyword}

\end{frontmatter}
\journal{ArXiv}
\newpage
\tableofcontents

\section{Introducción}
Cada ciencia intenta explicar un conjunto de fenómenos de la naturaleza mediante el método científico. El método científico asume que todo fenómeno de la naturaleza tiene una causa, a la que se denomina propiedad. Por ejemplo, la física propone que los fenómenos eléctricos son la manifestación de una propiedad denominada carga eléctrica. Así, la ciencia busca siempre hallar una teoría que explique el conjunto de fenómenos a partir de una relación entre los fenómenos y la propiedad que se asocia a ellos. Para que una teoría alcance el estatus de teoría científica, debe de ser posible realizar experimentos en los que se pueda comprobar la validez de sus postulados.

Uno de los objetivos que se plantea la ciencia para avanzar  en la explicación de los fenómenos de la naturaleza es el desarrollo de teorías de unificación. Una teoría de la unificación consiste en una teoría capaz de explicar mediante una sola propiedad y un conjunto de principios universales conjuntos de fenómenos que anteriormente eran considerados diferentes y sin relación. La primera unificación la llevó a cabo Newton mediante su obra Philosophiae Naturalis Principia Mathematica \cite{Newton}, donde unificó los fenómenos de los movimientos de los cuerpos celestes con los fenómenos de los movimientos de los cuerpos terrestres.

\subsection{ Las Ciencias Cognitivas}

Entre los conjuntos de fenómenos que hay en la naturaleza se encuentra el conjunto de los comportamientos cognitivos. Las ciencias cognitivas son un grupo de disciplinas que se dedican a estudiar los comportamientos cognitivos. Cada una de las disciplinas de las ciencias cognitivas intenta comprender y explicar un conjunto de comportamientos cognitivos concreto. Una de esas disciplinas es la Inteligencia Artificial (I.A.) que se ocupa del comportamiento cognitivo generado por máquinas. Otros ejemplos, son la etología y la neuroetología que tratan los comportamientos cognitivos generados por los animales. También está la psicología, neurofisiología y psiquiatría que se centra en los comportamientos cognitivos de los seres humanos. Para explicar un comportamiento cognitivo generado por un sistema (biológico o no biológico) cada ciencia cognitiva recurre a sus propias definiciones de propiedades o capacidades. En el caso de la psicología se habla de atención, memoria, percepción, etcétera. En la I.A. se habla de entropía en la organización de la información, el algoritmo de emparejamiento, los métodos con los que está construido el motor de inferencia, incertidumbre, etcétera. El abanico de disciplinas que en algunas de sus áreas estudia algún comportamiento cognitivo es enorme desde la lógica, pasando por la filosofía, hasta la lingüística. Además, hay que tener en cuenta que la interacción entre todas estas disciplinas es muy alta, por ejemplo entre inteligencia artificial y psicología, o entre inteligencia artificial y neurología, e incluso el surgimiento de nuevas áreas producto de estas interacciones como el área de psicología animal.

La I.A. nació en 1956 en la ``Dartmouth Summer Research Conference on Artificial Intelligence'', mayormente conocida como la conferencia de Dartmouth, con el objetivo de dotar de una completa inteligencia humana a los computadores. A la corriente científica que considera que es posible lograr el objetivo que se proponía en Dartmouth se la denomina I.A. fuerte. Durante todo este tiempo los investigadores no han cesado de trabajar para conseguir el objetivo de Dartmouth.  La idea más usada en la I.A. como guía para lograr el objetivo de Dartmouth ha sido desarrollar una teoría que explique los comportamientos humanos y que sea programable en un computador. A día de hoy, aunque el objetivo de Dartmouth no se ha conseguido, el trabajo de investigación que se ha realizado es enorme. Dentro de la I.A. se ha desa\-rrollado todo un campo denominado ingeniaría del conocimiento, en el que mediante metodologías instruccionales, las ciencias de la computación y las tecnologías de la información se ha logrado muy buenos resultados pese a des\-conocer los principios fundamentales que gobiernan los comportamientos en la naturaleza. Parece que en la I.A. se da un suceso, que ya ha ocurrido otras veces en la humanidad, y que consiste en que el descubrimiento de ciertos fenómenos es utilizado para desa\-rrollar una determinada tecnología; aunque las causas fundamentales que rijan el fenómeno se desconozcan. Un ejem\-plo de ello ocurrió con el telégrafo y el electromagnetismo, donde la técnica precedió a la ciencia.

Pero todos los trabajos desarrollados bajo el objetivo de Dartmouth y el propio objetivo pueden contemplarse con otro enfoque, el científico. Desde el enfoque científico, el trabajo que se ha realizado en la I.A. puede verse como el intento de encontrar teorías que unifiquen los comportamientos cognitivos de los seres humanos con los de los computadores. Los computadores no rea\-lizan innatamente comportamientos cognitivos, ya que son construidos por los seres humanos; pero eso no implica que no tengan un conjunto de comportamientos cognitivos asociado. El conjunto de comportamientos cognitivos de los computadores está compuesto por aquellos comportamientos para los que existe un programa que lo puede generar. Así, la primera propuesta de este artículo es reivindicar la aplicación del método que creó Newton para razonar sobre la naturaleza, que se halla en su obra Philosophiae Naturalis Principia Mathematica \cite{Newton}, a los comportamientos cognitivos.

Como se ha mencionado antes, el interés por desarrollar teorías que expliquen el comportamiento humano ha estado presente desde los inicios de la I.A..  En 1961,  Allan Newell  y H. A. Simon exploraron en su artículo\cite{NewellSimon1961} GPS, A  Program that Simulates Human Thought  en que grado se puede considerar a los programas informáticos como teorías capaces de explicar los comportamientos humanos. En el anterior artículo se puede leer:
 \newline
 \newline
``\emph {We will consider only how GPS behaves on the first part of the problem, and we will compare it in detail with the subject's behaviour as revealed in the protocol. This will shell considerable light on how far we can consider programs as theories of human problem solving}''.
 \newline

Los autores del citado artículo examinaron un programa desarrollado por ellos, el GPS\cite{GPS1959} (General Problem Solver). Los resultados que presentaron los autores en el artículo para considerar el programa GPS como  una teoría del comportamiento humano  fueron positivos, haciéndolo notar con una rotunda frase final:
 \newline
 \newline
``\emph{Although we know only for small fragments of behaviour, the deep of explanation is striking.}"
 \newline

 La cuestión de la unificación, aunque de una manera diferente a la pro\-puesta en este artículo, ha estado especialmente presente en la I.A. desde que Allan Newell presentó el problema de las microteorías \cite{Newell1981}. Newell mostró que las teorías que se estaban presentando, que eran correctas para los fragmentos de comportamientos escogidos, eran contradictorias entre sí. Así, Newell propuso a la psicología la necesidad de unificar las distintas teorías de los comportamientos humanos que se habían desarrollado. Esta idea la plasmó especialmente en su libro\cite{Newell1990} Unified Theories of Cognition. Newell, a diferencia de lo que se propone en este artículo, no consideraba que los computadores tuvieran un conjunto de comportamientos asociados, para él,  simplemente los computadores eran un mecanismo para comprobar si la unificación había tenido éxito. Así, Newell propuso, en su programa de unificación de teorías de la psicología, que las diversas teorías de unificación que se desarrollaran debían ser instanciadas en arquitecturas cognitivas, las cuales demostrarían mediante su funcionamiento en computadoras cuales eran correctas y cuáles no. Newell esperaba que se pudiera repetir continuadamente el proceso de unificación hasta llegar a una teoría que explicara todo el comportamiento humano.
 \newline
 \newline
``\textit{Thus, there will be multiple unified theories  for quite a while- if we can just get a few started, to show the way. My concern is that each one, itself, should be unified. We should be comparing one unified theory with another}''
 \newline

Hay que darse cuenta  que si la unificación de Newell tuviera éxito como consecuencia se lograría el objetivo de Dartmouth. Por desgracia, el programa de unificación para obtenerse finalmente una única teoría para el comportamiento humano que diseñó Newell no se está cumpliendo. Han trascurrido casi veinte años desde el comienzo del programa de Newell y la situación está detenida. Las distintas arquitecturas cognitivas \cite{ACT}\cite{brooks1986}\cite{SOAR}\cite{Copycat}\cite{kokinov94dual} que surgieron no consiguen demostrar que una sea más válida que las otras, por lo que entre estas teorías unificadas no ocurre la unificación predicha por Newell.

Otros intentos para lograr un mayor conocimiento de los comportamientos cognitivos han ido dirigidos a intentar unificar el conjunto de los comportamientos de los seres humanos con  determinados conjuntos de fenómenos de la física. Penrose\cite{Penrose1989} ha propuesto que las capacidades cognitivas humanas son consecuencia de las leyes de la física cuántica, aunque en su propuesta se excluye una unificación con los computadores actuales por carecer de mecanismos de computación cuántica que rompan el límite de la computación de Turing. Otra propuesta también en la misma línea, aunque esta propuesta sí incluye a los computadores es la de Doyle\cite{Doyle2006}. Doyle propone extender la mecánica clásica para explicar también los comportamientos que produce la mente humana.

Aunque las anteriores propuestas de unificación no han dado los resultados esperados, no todas las propuestas de unificación de las ciencias cognitivas han corrido la misma suerte. Un ejemplo de éxito en las ciencias cognitivas lo está dando la unificación entre la neurobiología y la psiquiatría, donde los últimos avances y descubrimientos científicos avalan fuertemente esta línea de trabajo\cite{Llinas1999}.

\subsection{ La Necesidad de una Definición Intensiva }

El conjunto formado por los comportamientos de los seres humanos es sin duda un conjunto de fenómenos para el que resulta muy interesante la pro\-puesta de desarrollar una teoría que logre explicarlos. Pero ese conjunto no es más que un subconjunto del conjunto que contiene todos los comportamientos. Durante el siglo veinte trabajos como los de Wolfgand Köhler y Edgard C. Tolman con chimpancés\cite{Kohler1925}, o los de Irene Peperberg \cite{Pepperberg1990}con loros grises han revelado que los comportamientos de otros animales no son triviales, haciendo que naciera la disciplina de la psicología animal.  Pero también el conjunto de comportamientos de todos los animales es un subconjunto de todos los comportamientos que pueden encontrarse en la naturaleza. Faltarían los comportamientos que son capaces de desarrollar el resto de organismos y sistemas no biológicos, como las plantas, robots e incluso las bacterias\cite{Tagkopoulos2008}. Por ejemplo, la planta de la judía es capaz de aprender ciclos de luz y mover sus hojas en función de estos ciclos, o la Drosera Byblis Drosophyllum, una planta carnívora, es capaz de reconocer si un alimento le vale o no, o incluso los comportamientos de los robots no pueden ser ignorados.

Pero la cuestión de unificar todos los conjuntos de comportamientos no es el único problema de unificación pendiente que tienen las ciencias cognitivas. Otro importante problema es la compatibilidad de las teorías de las ciencias cognitivas con las leyes físicas del nivel microscópico. Desde la postura científica se considera que la evolución del estado macroscópico de la naturaleza es consecuencia de las leyes que determinan la evolución del estado microscópico de la naturaleza. Pero ninguna de las teorías que se usan actualmente para explicar los distintos conjuntos de comportamientos enlaza con el nivel microscópico.

Con el objetivo de tratar de afrontar los dos problemas anteriores, es necesario dar un definición más precisa de los fenómenos a los que se hace referencia en este artículo, ya que el término  ``comportamientos cognitivos''\-  es impreciso y no conecta con el nivel microscópico de la naturaleza. Su imprecisión viene dada por dos motivos. El primer motivo es que el término ``cognitivos'' sólo hace referencia a un listado de comportamientos a los que se ha decidido dar este calificativo, por lo que si se presenta un nuevo fenómeno se debe decidir por consenso con el resto de la comunidad científica si se incluye en el conjunto de los comportamientos cognitivos; pero no hay una definición que permita decir si es un comportamiento cognitivo. El segundo motivo es que actualmente el término ``comportamientos'' se  emplea en muchos ámbitos de la ciencia, por ejemplo para los materiales, los sistemas dinámicos,... etc. Para eliminar la primera causa de imprecisión es necesario sustituir la definición extensiva del conjunto de los comportamientos cognitivos por una definición intensiva. La segunda causa de imprecisión se solventa  denotando a este conjunto de fenómenos mediante un nuevo término. El término elegido es ``exocomportamientos'' y la definición que fija los fenómenos a los que se va a hacer referencia en este trabajo se hace en función de la energía interna de un sistema, la cual pertenece al nivel microscópico de la naturaleza.

\begin{defNoNum}[Exocomportamiento]
Un exocomportamiento es la secuencia de cambios en el valor de las propiedades macroscópicas medibles de un sistema mediante su energía interna.
\end{defNoNum}

\subsection{Unificación de los Exocomportamientos}

 En consecuencia a lo mencionado, la segunda propuesta que se quiere plantear en este artículo es el problema de la unificación de los exocomportamientos. Este problema se enuncia de la siguiente manera:

\begin{defNoNum}[El problema de la Unificación de los Exocomportamientos]  Definido el conjunto de los exocomportamientos, ¿ existen un mismo conjunto de principios, conectado con el nivel microscópico de la naturaleza, que permita inferir la secuencia de actos en la que consiste cada exocomportamiento que realiza un sistema biológico o no biológico ?
\end{defNoNum}

Si se lograra encontrar una solución positiva al problema de la unificación de los exocomportamientos, esta consistiría en una teoría científica que explique los exocomportamientos de todos los organismos biológicos y todos los sistemas no biológicos y que enlace con el nivel microscópico de la naturaleza. Así, a diferencia de los anteriores retos de unificación de las ciencias cognitivas, el problema no sólo consiste en lograr una unificación entre conjuntos de exocomportamientos, sino establece una conexión con el nivel microscópico de la naturaleza.

Ahora bien, si se ha fracasado hasta ahora en el intento por construir una teoría para el conjunto de exocomportamientos de los seres humanos, ¿por qué se habría de tener éxito en el problema de la unificación de los exocomportamientos?  A esta pregunta hay una doble respuesta. Lo primero que se debe de explicar, es que lo que verdaderamente ha fracasado ha sido el intento de usar analogías. John Doyle en su libro\cite{Doyle2006}  Extending Mechanics to Minds realiza en el capítulo tercero un examen  de las analogías que se han usado en las ciencias cognitivas con el objetivo de permitir al científico comprender y razonar sobre el pensamiento humano. Doyle manifiesta que cada intento de usar una analogía (biológica, antropológica, teológica, química, dinámica, energética, termodinámica y de máquina) para fundamentar el comportamiento cognitivo ha fracasado. Las analogías pueden ser útiles como primera aproximación cuando la ciencia se acerca a un fenómeno desconocido; pero es difícil que esa aproximación, que se realiza mediante una extrapolación de un campo de la naturaleza hacia otro, sea una sólida explicación. Valga de ejemplo la analogía que usaron los pitagóricos cuando intentaron explicar con las proporciones que dictaban la generación de música en los instrumentos el movimiento de los planetas. Sin duda, la explicación de que los planetas se mueven para cumplir las proporciones nada tiene que ver con la teoría de Newton y ni que decir con la de Einstein. Otro ejemplo de esto, nos lo brinda también la física cuando se usaron analogías para desarrollar los modelos de los átomos, primero con el del pudin de Thomson y luego con el modelo gravitatorio de movimiento de los planetas para el mo\-delo del átomo de Rutherford. Sin duda la naturaleza es tan peculiar que las analogías de poco más que de primeras aproximaciones pueden servir cuando se cambia de un aspecto a otro de la naturaleza.

Lo segundo, es que es cierto que no se han encontrado unos principios comunes que expliquen todos los comportamientos de los seres humanos; pero la causa de no encontrarse esos principios podría residir en el propio problema. El motivo de esta afirmación es debida a que no se puede identificar la causa para un conjunto de fenómenos si el conjunto de fenómenos es un subconjunto de todos los fenómenos que tienen una misma causa y en la búsqueda de la causa se mantiene la suposición errónea de que el conjunto no es un subconjunto. La razón es que la causa que se proponga se desechará debido a que existirán fenómenos que se puede explicar con la causa propuesta pero que caerán fuera del subconjunto. Este caso se puede dar cuando la definición del conjunto de fenómenos que se quiere explicar se hace en base a la experiencia sensible. Por ejemplo, la definición popular de calor que está basada en la experiencia sensible no acepta que un cubito de hielo pueda generar calor, así muchos fenómenos quedan sin poder ser explicados con razonamientos basados en la experiencia sensible. En cambio, la física ha definido en términos absolutos de la naturaleza un concepto denominado calor. Según el concepto de calor de la física un cubito de hielo pueda generar calor, a pesar de que esa noción vaya contra la experiencia sensible del ser humano. Así, a cambio de aceptar la noción de calor definida por la física se logra tener razonamientos para explicar fenómenos que antes no se podían desde la experiencia sensible. Por lo tanto, el elegir un conjunto de fenómenos en base a la experiencia sensible conduce a la imposibilidad de diferenciar principios fundamentales que posee la naturaleza de consecuencias de estos. De esta manera, podría ser que el conjunto de exocomportamientos humanos sea sólo explicable cuando se posea una teoría para explicar todos los exocomportamientos de la naturaleza.

Pero este documento no se detiene en la presentación del problema de la unificación de los exocomportamientos, sino que se va a presentar una posible solución, la Teoría General del Exocomportamiento (TGE). La TGE, que es la tercera propuesta del documento, no se trata de ningún tipo de analogía, sino de una teoría científica que propone que los comportamientos son la manifestación de una propiedad emergente de la naturaleza.

\section{Definiciones y Postulados de la TGE }

En el razonamiento científico siempre se asume que los fenómenos son la manifestación de propiedades que existen en la naturaleza. Entre el concepto de fenómeno y propiedad existe una dualidad que a veces hace que se haga un abuso de los términos, y sean intercambiados al hablar sobre aspectos de la naturaleza. Así que a continuación, se intentará dejar claro cada uno de estos dos conceptos y su relación. Los fenómenos son cambios que se producen en las propiedades medibles de los sistemas. La ciencia propone que esos fenómenos son consecuencia de la existencia de propiedades intrínsecas en los sistemas, y por lo tanto, los fenómenos delatan la existencia de propiedades.  Puesto que las propiedades son intrínsecas no hay una mane\-ra directa de conocer (medir) las propiedades, pero sí puede ser conocida (medida) de manera indirecta a través del fenómeno. Ejemplos de esta dualidad entre propiedad y fenómeno en física son temperatura y calor, carga eléctrica y fuerza eléctrica, o masa y gravedad.

Una vez se ha definido una propiedad, debe de  enunciarse un conjunto  finito de hechos fundamentales y universales que relacionen la propiedad con los fenómenos. A cada una de las afirmaciones de ese núcleo de hechos se le denomina postulado. Los hechos que se enuncien serán fundamentales si a partir de ellos se derivan otros hechos de la naturaleza; pero ellos no deben de poder ser derivados de otros. Los postulados son universales porque deben en cumplirse siempre en la naturaleza. Sin entrar en profundidad en los conceptos del método científico y la filosofía de la ciencia, se ha de mencionar que los postulados de una teoría científica no pueden probarse que son hechos universales de la naturaleza; pero se asumen como tales, ya que siempre que se han puesto a prueba en la naturaleza no han podido falsearse. Puesto que los postulados son considerados hechos fundamentales, sobre ellos se puede aplicar un proceso deductivo que permite inferir nuevo conocimiento sobre la naturaleza, o explicar hechos de la natu\-raleza. La explicación que aporta una teoría científica sobre la naturaleza sigue vigente mientras los postulados de la teoría no se falseen y las consecuencias que se deducen de ellos se ajusten a lo que ocurre en la naturaleza.

En esta sección se establecerán las bases de la TGE. Primero se definirán una serie  de términos y relaciones elementales. Entre los términos fundamentales se definirá el conjunto de fenómenos que se quiere llegar a explicar, los exocomportamientos, y una propiedad, la fasa, que se fijará como causa de ese conjunto de fenómenos. A conti\-nuación, se establecerán los postulados de la teoría. Los últimos apartados de la sección se dedicarán a presentar consecuencias que se pueden deducir de las definiciones, como el teorema fundamental de los comportamientos.

Para que los postulados puedan ser interpretados correctamente se deben fijar las definiciones de los términos y relaciones  que aparecen en ellos. Así,  este apartado será usado para fijar esas definiciones, y algunas otras que se usarán durante todo el artículo.

\subsubsection {Términos Primitivos de la TGE}

Los términos primitivos que maneja la TGE son los siguientes:

\begin{definicion}[Sistema] Un sistema es un agregado de objetos físicos entre cuyas partes existe una relación. La relación que define el sistema puede ser física o lógica. Puede tener complejidad biológica o no.
\end{definicion}

\begin{definicion}[ Subsistema] Un subsistema es un sistema que pertenece a un conjunto de sistemas entre los que existe una relación que los liga como un sistema.
\end{definicion}

\begin{definicion}[ Estado de un sistema (Subsistema)] Un estado de un sistema es cada una de las formas físicamente distinguibles que puede adoptar un sistema (subsistema).
\end{definicion}

\begin{definicion}[ Universo de un sistema $\mathbf{S}$] El universo de un sistema $\mathbf{S}$ es el  sistema cerrado que contiene al sistema  $\mathbf{S}$.
\end{definicion}

\begin{definicion}[ Estado realidad del universo en el instante $\mathbf{t}$] El estado realidad del universo en el instante $\mathbf{t}$ es el estado del universo en el que se encuentra el universo en el instante de tiempo $\mathbf{t}$. Al estado realidad del instante t se le denotará por la etiqueta $\mathbf{r}_{\mathbf{t}}$.
\end{definicion}

\begin{definicion}[Exocomportamiento de un sistema $S$]Un exocomportamiento de un sistema $S$ es la secuencia de cambios en el valor de las propiedades macroscópicas medibles del sistema $S$ mediante su energía interna ordenados en el tiempo.
\end{definicion}

\begin{definicion}[ Acto de un sistema $S$ en un instante $t$] El acto de un sistema $S$ en un instante $t$ es un proceso que cambia el valor de las propiedades macroscópicas medibles del sistema $S$ en el instante $t$ mediante la energía interna del propio sistema $S$ .
\end{definicion}

\begin{definicion}[ Fasa] La fasa es una propiedad de los sistemas que se manifiesta en los exocomportamientos de los sistemas. El valor de la fasa de un sistema $S$ no es numérico sino que es descrito por una función, a la que se denotará por $f_{s}$.
\end{definicion}

\begin{definicion}[ Exocomportamiento posicional] Un sistema lleva a cabo un exocomportamiento posicional cuando el valor de la fasa del sistema es únicamente de carácter  posicional.
\end{definicion}

\begin{definicion}[ Fasa de carácter posicional] El valor de la fasa de un sistema es de carácter posicional  cuando cada acto de la secuencia es función de su posición en la secuencia. Se denotará por  $f_{s}^{P}$.
\end{definicion}

\begin{definicion}[ Exocomportamiento aleatorio] Un sistema lleva a cabo un exocomportamiento aleatorio cuando el valor de la fasa del sistema es únicamente de carácter  aleatorio.
\end{definicion}

\begin{definicion}[ Fasa de carácter aleatorio] El valor de la fasa de un sistema es de tipo aleatorio cuando cada acto de la secuencia es función de una variable aleatoria independiente del estado del universo. Se denotará por  $f_{s}^{A}$. Si la variable aleatoria puede ir cambiando entonces cada variable aleatoria de la secuencia que forman las variables aleatorias en el tiempo debe cumplir que sea seleccionada a su vez por una variable aleatoria independiente del estado del universo.
\end{definicion}

\begin{definicion}[ Exocomportamiento sensible] Un sistema lleva a cabo un exocomportamiento sensible cuando el valor de la fasa del sistema es únicamente de carácter  sensible.
\end{definicion}

\begin{definicion}[ Fasa de carácter sensible] El carácter sensible de la fasa de un sistema es la capacidad de la fasa para generar una secuencia de actos, donde cada acto que contiene la secuencia es función del estado realidad del universo. Se denotará por  $f_{s}^{S}$.
\end{definicion}

\begin{definicion}[ Universo con sensibilidad a las condiciones iniciales] Un universo tiene sensibilidad a las condiciones iniciales si sus estados independientemente de su similitud seguirán evoluciones significativamente diferentes.
\end{definicion}

\begin{definicion}[ Exoactivo] Un sistema está en un estado exoactivo cuando el sistema es capaz de realizar un exocomportamiento. Cuando un sistema es incapaz de realizar un exocomportamiento  se dice que el sistema esta exoinactivo.
\end{definicion}

\begin{definicion}[ Restricciones de exoactividad de un Universo] Las restricciones de exoactividad de un universo es el conjunto de reglas que dictan cuando un sistema pasa de un estado exoactivo a estar exoinactivo.
\end{definicion}

La definición de exocomportamiento permite definir un conjunto de fenómenos donde se encuentran desde los comportamientos de los seres humanos hasta los de los robots, pasando por plantas y bacterias. Por ejemplo, una bacteria, la Escherichia coli, realiza movimientos atractivos y repulsivos. La bacteria altera su posición en el espacio a costa de una pérdida de energía interna, por lo tanto las alteraciones de la posición de la bacteria sí son un exocomportamiento. Pero se debe notar que según la definición de exocomportamiento no toda secuencia de alteraciones de las propiedades observables de un sistema es un exocomportamiento. Recuérdese que la energía de un sistema físico es la suma de su energía potencial externa, su energía cinética externa, su energía potencial interna y cinética interna. A continuación, se expondrán los tipos de fenómenos que no pertenecen al conjunto de los exocomportamientos:

\begin{itemize}
  \item FENÓMENOS DE TRANSFORMACIÓN DE ENERGÍA POTENCIAL EXTERNA A CINÉTICA EXTERNA.
  \newline
  Los fenómenos en los que hay una transformación de la energía potencial externa de un sistema en energía potencia interna no son fenómenos que entren dentro del conjunto de los exocomportamientos. El movimiento de una partícula es el resultado de la transformación de la energía potencial externa que posee la partícula debida a los campos eléctricos que generan el resto de partículas del universo en energía cinética externa.

    \item FENÓMENOS DE TRANSFORMACIÓN DE ENERGÍA CINÉTICA EXTERNA A  POTENCIAL EXTERNA.
    \newline
  Los fenómenos en los que hay una transformación de la energía energía potencial interna de un sistema en  externa no son fenómenos que entren dentro del conjunto de los exocomportamientos.

  \item FENÓMENOS DE INTERCAMBIO DE ENERGÍA CINÉTICA EXTERNA.
  \newline
  Los fenómenos de intercambio de energía externa no son exocomportamientos. Por ejemplo, la alteración de la trayectoria de un cuerpo por el choque con otro cuerpo. Este fenómeno tampoco pertenece al conjunto de los comportamientos. La razón es que la alteración se produce por las energías cinéticas externas del sistema.

  \item FENÓMENOS DE TRANSFORMACIÓN DE ENERGÍA POTENCIAL INTERNA A  CINÉTICA INTERNA.
  \newline
  Los fenómenos en los que se transforma energía potencial interna en energía cinética interna, por ejemplo la corriente eléctrica, no es un fenómeno que entre dentro del conjunto de los exocomportamientos. Pero el intercambio de calor que realiza el sistema con el universo producido por la corriente eléctrica si es un exocomportamiento.

  \item FENÓMENOS DE TRANSFORMACIÓN DE ENERGÍA EXTERNA A  ENERGÍA INTERNA.
  \newline
  El fenómeno en el que un sistema adquiere energía interna no es un exocomportamientos, bien sea por trabajo o calor.

  \item FENÓMENOS SIN TRANSFORMACIÓN DE ENERGÍA.
  \newline
   Las fenómenos en los que no hay ninguna transformación de energía interna en energía cinética externa no son exocomportamientos. Por ejemplo, un movimiento rectilíneo uniforme sin rozamiento no es un exocomportamiento, ya que un cuerpo que se mueve en un movimiento rectilíneo uniforme sin resistencia no usa su energía interna para mantener la velocidad constante.

\end{itemize}

Otra cuestión relevante sobre las definiciones que se han establecido, es que siguiéndose el razonamiento científico se ha postulado la existencia de una propiedad que causa los comportamientos. Por lo tanto, existe una relación entre cada uno de los exocomportamientos y  una propiedad denominada fasa. Esa relación toma la expresión ecuacional siguiente:

\[
\varepsilon_{s}(t)=\rho f_{s}(p_{0},...,p_{m})   \qquad m\geq 0
\]

\begin{itemize}
  \item $\varepsilon_{s}$ es el exocomportamiento del sistema-s.

  \item $\varepsilon_{s}(t)$  es el acto que realiza el sistema-s en el instante del tiempo t.

  \item $\rho$ es la relación que existe entre la fasa y el exocomportamiento de un sistema.

  \item $f_{s}$ es el valor  que tiene la propiedad fasa del sistema-s.

  \item $p_{0},...,p_{m}$ son los valores que toman los parámetros de la función  $f_{s}$.

\end{itemize}
Se debe tener en cuenta que, por la definición de los distintos tipos de carácter de la propiedad fasa, los tipos de carácter no se anulan entre sí; a diferencia de lo que ocurre con los tipos de carácter de otras propiedades de la naturaleza como la carga eléctrica.

\subsubsection { Relaciones Primitivas de la TGE}

A continuación, se definen dos relaciones fundamentales para la TGE que relacionan términos primitivos definidos en el apartado anterior.

 \begin{definicion} [Representación] Un sistema, $S$, tiene una representación si el sistema $S$ contiene un subsistema $R$ y existe una función, $r$, cuyo dominio es el conjunto de estados del universo $U$ y su codominio es el conjunto de estados del subsistema $R$. El número de estados del subsistema $R$ que conforman la imagen de la función $r$ tiene que ser mayor o igual que $2$.
 \end{definicion}

En un primer momento podría considerarse que el dominio de la función $r$  debería ser un conjunto de subestados y no de estados, argumentándose que con la definición actual, si $S$ modificara su posición sin que hubiera modificaciones en el resto de sistemas la definición obligaría a asociar un nuevo estado de $R$ al mismo estado del dominio, con lo que r dejaría de ser una función. Pero ese razonamiento es incorrecto, ya que el estado del universo incluye al sistema $S$, de modo que, si $S$ modifica su posición, el universo se encuentra en otro estado, y es a ese nuevo estado al que se le asocia el nuevo elemento del codominio.

\begin{definicion} [ Mecanismo de Computación]Un sistema, $S$, tiene un mecanismo de computación: Un sistema, $S$, tiene un mecanismo de computación cuando tiene un subsistema, $C$, y posee un proceso que al aplicarse sobre un estado de su subsistema $C$ lo transforman en otro estado. El proceso es equivalente a una función que tiene como dominio y codominio los estados del subsistema $C$.
\end{definicion}

Como se puede observar en la anterior definición, la computación se ve como una función que manda los elementos del dominio a sus respectivas imágenes en el codominio. Así, cuando la TGE habla de computación lo hace usando la  noción de función. Este enfoque funcional fue usado originariamente en los trabajos de Kurt Gödel, Alonzo Church, Alan Turing y Stephen Kleene. La ventaja del enfoque funcional para definir computación es que permite abstraerse de un mecanismo de computación concreto, lo que permite que los resultados que se alcancen sean independientes del mecanismo computacional, ya sea una maquina de Turing, una red neuronal,...,etc. Este enfoque se usó para desarrollar la teoría de la recursividad, pero a diferencia del concepto de computabilidad de esa teoría, en la TGE se identifica computabilidad completamente con la noción de función, esto implica que se considera también computación a funciones que no son computables por una máquina de Turing, lo que se conoce por hipercomputación\cite{Copeland1999}.

\subsection{ Postulados de la TGE }

A continuación, se van a presentar los dos postulados que constituyen el núcleo de la TGE. La cuestión del desarrollo de experimentos que permitan poner los postulados a prueba se tratará en la sección 5.

\subsubsection{Postulado primero} Según las definiciones que se han establecido existen tres conjuntos de comportamientos elementales; pero se desconoce si en el conjunto de los exocomportamientos podría existir otro subconjunto de exocomportamientos elementales diferente a los existentes. Si existiera otro conjunto de exocomportamientos elementales en las definiciones entonces ese hecho delataría la existencia de otro tipo de carácter para la propiedad fasa. El siguiente postulado aclara esa cuestión.

\begin{postulado}[Postulado de los exocomportamientos elementales]
En la naturaleza no existen más de tres tipos elementales de exocomportamientos: los exocomportamientos posicionales, los exocomportamientos aleatorios y los exocomportamientos sensibles.
\end{postulado}

El significado de este postulado puede ser visualizado muy fácilmente con diagramas de conjuntos.  El modo de visualizarlo consiste en una superficie negra en la que se representa cada comportamiento por un punto con un color.  Los exocomportamientos elementales se pintan con un color primario: Un exocomportamiento aleatorio se pintan con un punto de color rojo, un exocomportamiento posicional  con un punto de color verde y un exocomportamiento sensible de color azul. Los  diagramas que contienen comportamientos aleatorios, posicionales y sensibles son disjuntos entre ellos. Por otro lado el resto de puntos tiene un color que es combinación de los tres colores primarios, representando que ese exocomportamiento es combinación de los tres tipos de carácter de la fasa. El diagrama que aparece se muestra en la figura \ref{fig:post1}.
\begin{center}
\begin{figure}
  \centering
    \includegraphics{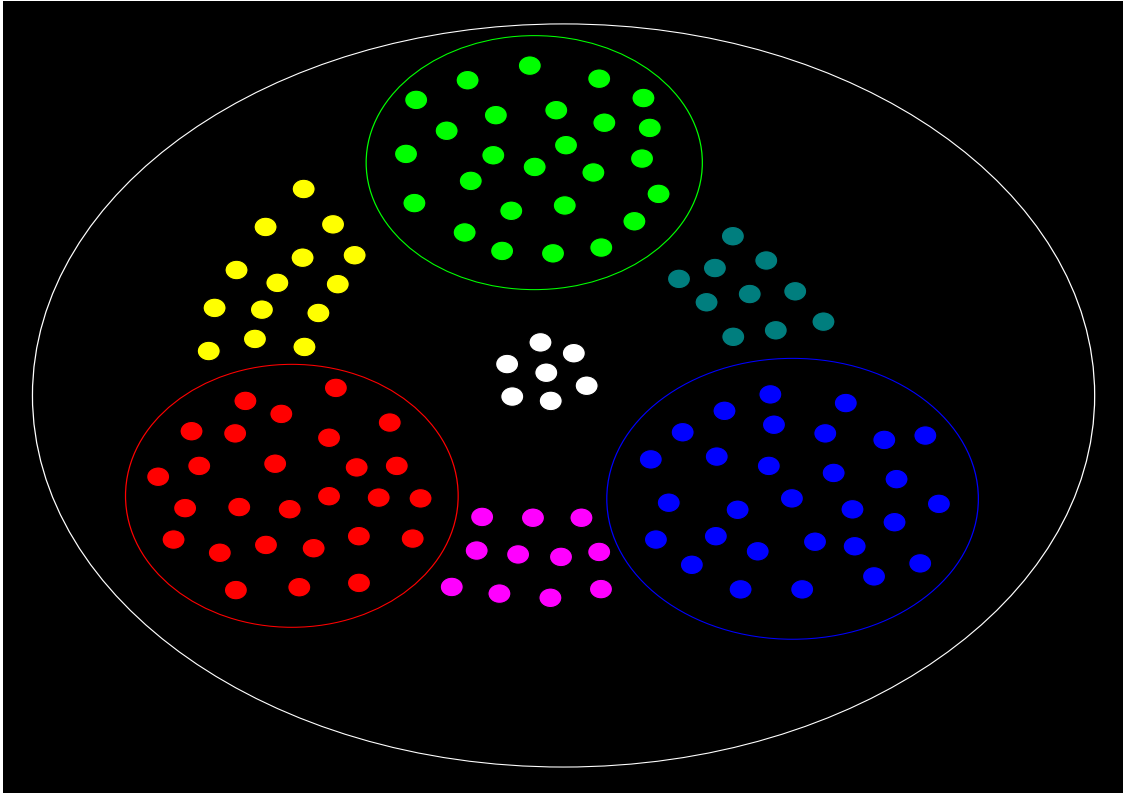}
  \caption{En el diagrama cada exocomportamiento se representa con un punto cuyo color expresa el tipo de exocomportamiento. Los exocomportamientos elementales tienen asignados los colores primarios. Los exocomportamientos aleatorios se pintan con el color rojo, los posicionales con verde y los sensibles con azul. Los conjuntos de exocomportamiento elementales son disjuntos entre ellos.}
  \label{fig:post1}
\end{figure}
\end{center}

Antes de terminar este apartado se debe aclarar una cuestión para comprender exactamente cuáles son los elementos que contiene el conjunto de exocomportamientos sensibles. La cuestión gira en torno a que no se debe confundir el mecanismo que realiza los actos con el mecanismo que decide los actos, ya que, es el mecanismo de decisión de un sistema, y no el mecanismo de actuación, el que se usa para clasificar el comportamiento del sistema. Para eliminar cualquier duda, tómese de ejemplo, la Escherichia coli, una bacteria perítrica. Esta bacteria perítrica usa un mecanismo de actuación de tipo aleatorio debido a que su locomoción se basa en procesos estadísticamente sesgados de modo que el resultado final le permite poder desplazarse en una dirección u otra. En la TGE el exocomportamiento de una bacteria perítrica, como la citada anteriormente, es catalogado como exocomportamiento sensible. La razón es que dado un concreto entorno químico la distribución de probabilidad sesgada es una concreta, ya que las distintas distribuciones están fijadas a las sustancias químicas que se detectan. Por lo tanto, el método de decisión no elige  la distribución de probabilidad, que permite a la bacteria moverse en una dirección, ni de una manera aleatoria, ni preestablecida.

\subsubsection{Postulado segundo}

\begin{postulado}[Postulado de la Representación]
En universos con sensibilidad a las condiciones iniciales y restricciones de exoactividad, los sistemas cuyos exocomportamientos sean posicionales, o aleatorios terminarán exoinactivos, y sólo lograrán permanecer exoactivos los sistemas donde los actos que componen sus exocomportamientos se producen en función de una representación del estado realidad del universo que se adecua a las restricciones de exoactividad.
\end{postulado}

Este segundo postulado es una extensión de las reflexiones y propuestas del profesor Rodolfo Llinás\cite{Llinas1987}\cite{Llinas2001}, uno de los padres de las neurociencias, para dar respuesta  a la pregunta de cuál es la razón de la aparición del cerebro en los seres vivos. Llinás ha sugerido la hipótesis de que el cerebro ha surgido debido a que los peligros que trae consigo moverse son determinantes. Así, el organismo que se mueve sin usar una representación de su entorno, o él que usa una representación incorrecta, se extingue rápidamente. Por lo tanto, se puede decir que, entre los organismos con capacidad de moverse la selección natural elige aquellos que hacen uso para moverse de una representación correcta del estado de su entorno.

El profesor Llinás sostiene esa hipótesis conforme a ciertos hechos del mundo de la biología. Uno de estos hechos son los tunicados. Los tunicados son unos seres que viven toda su fase adulta sujetos a un objeto estable en el mar. El tunicado en fase adulta lleva a cabo dos funcione básicas: Alimentarse mediante el filtrado del agua de mar y reproducirse. La larva del tunicado tiene durante una cantidad de tiempo pequeña (un día o menos) capacidad para nadar libremente y un pequeño cerebro. El primitivo sistema nervioso del tunicado permite recibir información sensorial sobre el entorno que le rodea, de manera que sabe lo que es arriba, abajo, izquierda, derecha.... Cuando la larva encuentra un lugar adecuado se posa y queda sésil en ese lugar.
Una vez ocurre esto, la larva pasa al estado adulto. En este proceso la larva digiere literalmente gran parte de su cerebro, dejando tan sólo lo necesario para la simple actividad de filtrado de agua. La conclusión que saca el profesor Llinás de este hecho es clara: ``el cerebro se desarrolla para que el ser vivo que se mueve comprenda el entorno que le rodea y pueda moverse autónomamente sin perecer''. La capacidad de moverse es muy positiva pero Llinás también explica que es algo muy peligroso. Si un ser se mueve pero no sabe donde se mueve puede terminar en la boca del que se lo va a comer. El tunicado es la pieza clave en la teoría de la evolución de los cordados, ya que es a partir de su estado larvario del que evolucionan los cordados. Así pues, la razón del tunicado es la razón principal por la que haya seres vivos con cerebro, esto es, para comprender su entorno.

Si se piensa durante un instante a qué conjunto de exocomportamientos de la clasificación de la TGE pertenecen los exocomportamientos de los cordados se obtiene que se trata de exocomportamiento sensible. Eso es debido a que los actos que realizan los cordados son función de la representación que poseen en su cerebro del estado de su entorno.  Ahora bien, el problema de la unificación de los exocomportamientos no consiste tan sólo en los exocomportamientos de los cordados, sino el de cualquier sistema. Por lo tanto, la hipótesis el profesor Llinás debe ser generalizada para que abarque el exocomportamiento de cualquier sistema y sin que contradiga el propio enunciado del profesor Llinás.

La extensión se hace en base a la definición de representación que contiene la TGE, ya que la definición de representación es independiente del mecanismo concreto, y se basa en que exista una función que asocie los estados del ambiente con los de un subsistema. Así, la TGE considera que el cerebro es uno de los mecanismos que la naturaleza puede usar para albergar una representación del entorno pero no el único. El dominio del mecanismo del cerebro frente a otros mecanismos de la naturaleza para albergar una representación vendría dado por sus mejores características, ya sean energéticas, computacionales o una mezcla de ambas.  El trabajo del profesor Llinás dice que, los organismos con capacidad de moverse y cerebro son seleccionados frente a los organismos con capacidad de moverse que no tienen una representación. Por lo tanto, que otros sistemas tengan un mecanismo para albergar una representación diferente del cerebro y sean también seleccionados por la selección natural junto a los que albergan la representación en el cerebro no contradice la hipótesis del profesor Llinás y logra generalizarla.

Existen dos cuestiones más que deben quedar claras sobre el postulado. La primera es que cuando se usa la expresión \emph {``es función de''}, se hace referencia a la definición de mecanismo computacional, donde se fijaba que en la TGE se identificaría el concepto de mecanismo de computación con el de función. Por lo tanto, el postulado fija que el sistema que realiza el comportamiento sensible posee un mecanismo de computación que usa una representación del universo para obtener la secuencia de actos  del comportamiento.

La otra cuestión que debe de conocerse es que cuando en el postulado se dice \emph {``representación del estado realidad del universo que se adecua a las restricciones de exoactividad'' } significa que la representación contiene suficiente información para persistir y no hay información incorrecta sobre las restricciones de exoactividad. Más adelante este tema será tratado de manera más precisa.

\subsection{ Teorema Fundamental de los Exocomportamientos}

En este apartado se va a presentar un resultado que se deriva de las definiciones de la TGE, y al que se recurrirá en posteriores secciones, ya que su demostración llevará a establecer relaciones matemáticas que contienen las definiciones de los distintos tipos de carácter de la fasa de un sistema y su exocomportamiento. La razón de la necesidad de este teorema es que el postulado primero sólo fija que en la naturaleza no hay más de tres tipos de carácter de la fasa, pero no se sabe si podría haber menos de tres tipos de carácter. Eso podría ocurrir en el caso de que para alguna de las definiciones no hubiera ningún exocomportamiento que la cumpliera.

\begin{theorem}[Teorema fundamental de los exocomportamientos]
Ninguno de los conjuntos de exocomportamientos elementales es vacío.
\end{theorem}

Este teorema junto al postulado primero establecen que existen sólo tres tipos de carácter de la fasa y que sólo existen tres conjuntos de exocomportamientos elementales. Para demostrar el teorema hay que demostrar que ninguno de los conjuntos de exocomportamientos elementales que se define es vacío. Para demostrar que un conjunto es no vacío basta con encontrar un elemento que pertenezca al conjunto. Es decir, al menos existe un sistema que lleva a cabo un exocomportamiento que pertenece al conjunto definido. Para demostrar este teorema hay que demostrar tres lemas. Cada uno de los lemas dice que existe al menos un elemento en su correspondiente conjunto de exocomportamientos elementales. Ahora bien, para probar de qué tipo es un exocomportamiento se debe conocer perfectamente cómo surge el exocomportamiento. La manera más sencilla de tener esa información es crear el exocomportamiento. Eso se puede lograr programando un robot. Así, la prueba del teorema se infiere directamente de los tres lemas que se van a demostrar a continuación.  Para demostrar cada lema se indicará como crear un programa que genere el exocomportamiento elemental correspondiente.

\begin{lemma}
El conjunto de exocomportamientos aleatorios es no vacío
\end{lemma}

\begin{proof}
Según la definición de exocomportamientos aleatorios y fasa aleatoria su relación se puede expresar de la siguiente manera:

\[
\varepsilon_{s}(t)=\rho f^{A}_{s}(X_{t})
\]

Donde  $\varepsilon_{s}(t)$ es una función que dicta que acto realiza en cada momento del tiempo el sistema-s, de manera que siendo $\mathbf{T}$ el conjunto de instantes del tiempo y $\mathbf{A}^{s}$ los actos que puede realizar el sistema-s, la función se define como:

\[
\begin{array}{cccc}
  \varepsilon_{s} & \mathbf{T} &\longrightarrow &\mathbf{A}^{s} \\
   & t & \mapsto & \mathbf{a}_{x}
\end{array}
\]

 La función $f^{A}_{s}$ es una función aleatoria y $X_{t} $ es una variable aleatoria independiente del estado del universo. Así, fijándose en la ecuación se puede crear un programa informático que use un generador de números aleatorios para decidir la acción que tiene que ejecutar en cada momento del tiempo el robot. Así, cuando el robot que ejecute el programa presentará un exocomportamiento aleatorio.

\end{proof}

\begin{lemma}
El conjunto de exocomportamientos posicionales es no vacío
\end{lemma}

\begin{proof}
Según la definición de exocomportamientos posicionales y fasa posicional, la secuencia de actos de un exocomportamiento se puede expresar matemáticamente de la siguiente manera:

\[
\langle (\rho f^{P}_{s}(1))_{1},...,(\rho f^{P}_{s}(n))_{n}   \rangle
\]

Así, fijándose en la expresión matemática del comportamiento se puede crear un programa informático que genere una secuencia de acciones. Por ejemplo, si se tiene un robot que puede ejecutar diez actos diferentes, entonces se crea un programa que genera la secuencia de dígitos que componen un número, se elige por ejemplo el número pi. Entre los dígitos y los actos se fija una función exhaustiva, $f^{P}_{s}$. Así, cada vez que el programa genere un determinado dígito  el robot ejecuta la acción asociada a ese dígito. Si se quiere generar un comportamiento posicional  para un robot que tenga un número de actos distinto de diez, basta con cambiar la base en la que se trabaja para la representación de los números.

\end{proof}

\begin{lemma}
El conjunto de exocomportamientos sensibles es no vacío
\end{lemma}

\begin{proof}
Según la definición de exocomportamientos sensibles y fasa sensible su relación se puede expresar de la siguiente manera:

\[
\varepsilon_{s}(t)=\rho f^{S}_{s}(\psi_{r},p_{1},...,p_{m})
\]

La fórmula  $\psi_{r}$  es una representación del estado realidad en el momento $t$ del tiempo, y los parámetros  $p_{1},..., p_{m}$ aparecen debido a que el teorema dice que $\varepsilon_{s}(t)$  es función de  $\psi_{r}$ pero  no dice en función únicamente de  $\psi_{r}$, por lo que podría haber más parámetros, aunque estos  no pueden ser ni una variable aleatoria independiente del estado del universo, ni la posición del acto en la secuencia . Así, $m\geq 0$.

Por lo tanto, como consecuencia de la expresión matemática es necesario desarrollar un programa que tome representaciones del estado del entorno y le asocie un acto del robot. Así,  por ejemplo, si se desarrolla un programa que capte de la cámara de visión artificial de un robot y le asocie a cada estado captado una acción concreta del robot, ocurrirá que al ejecutarse el programa en el robot, el comportamiento del robot será un comportamiento sensible, ya que sus actos sólo son función del estado en el que se encuentre el universo.

\end{proof}

 Este  teorema junto al postulado primero tiene una importante implicación para el desarrollo de la TGE. Ya que como consecuencia directa de ellos el desarrollo de la TGE pasa en primer lugar por explicar los tres tipos de comportamientos elementales. Para describir de manera precisa cada uno de ellos es necesaria una teoría matemática que permita describirlos. Así, la TGE contiene  una teoría para describir los exocomportamientos aleatorios,  otra teoría para describir los exocomportamientos posicionales  y una  última teoría para los exocomportamientos sensibles. Pero de las tres teorías, tan sólo es necesario desarrollar la que se ocupe del conjunto de comportamientos sensibles. La razón es que para los otros dos conjuntos de comportamientos ya existen ramas matemáticas que los pueden tratar. En el caso de los comportamientos posicionales  es la teoría de números la que se debe de encargar de hallar el valor de la fasa del sistema. Esto es una consecuencia de que la secuencia de actos puede ser interpretada como un número, donde cada dígito del número representa el acto que lleva a cabo el sistema. Para el caso de los comportamientos aleatorios, puesto que la secuencia de actos se trata de un proceso estocástico, es la teoría de la probabilidad la encargada de hallar la función aleatoria que genere el comportamiento y que es el valor de la fasa aleatoria.

\subsection{Otros Resultados sobre la Fasa}

De las definiciones de los términos primitivos realizadas en la teoría que se está presentando se pueden extraer ciertas consecuencias sobre la propiedad fasa de un sistema. Esas consecuencias se enuncian en los siguientes resultados.

\subsubsection{Características la Fasa}

\begin{corolario}[Características de la fasa]
Los sistemas que tiene fasa tiene mecanismos computacionales, bien sean deterministas o  bien probabilistas, con capacidad clásica o de hipercomputación.
\end{corolario}

Este resultado surge de la definición de fasa y de sistema de computación. La definición de fasa implica la necesidad de un mecanismo computacional, bien sea determinista o probabilísta. El mecanismo computacional, al ser descrito por el concepto de función engloba tanto a las funciones computables de manera clásica, como a las no computables de manera clásica. Es decir, el mecanismo puede ser un mecanismo clásico o de hipercomputación.

\subsection{Exocomportamiento Sensible y Fasa Sensible}
Como se ha mencionado anteriormente, el primer paso que se va a dar en el desarrollo de la TGE es explicar el conjunto de exocomportamientos sensibles. Por lo tanto, este apartado se dedicará a tratar cuestiones referentes al exocomportamiento sensible y la fasa de carácter sensible.

\subsection{Significado de las Ecuaciones del Exocomportamiento Sensible.}
Los fenómenos de la naturaleza pueden ser medidos directamente; pero los valores de la propiedad que los produce no. Así, se necesita usar un método indirecto que permita obtener el valor de la propiedad. Este método se basa en usar la relación entre fenómeno y propiedad. Esa relación se describe mediante ecuaciones, y por lo tanto, resolver las ecuaciones es el método que permite hallar los valores que puede tomar la propiedad. Como en este caso los valores de la propiedad que se intentan conocer se corresponden con las funciones que  pueden sustituir  $f_{s}^{S}$, las ecuaciones que contiene la teoría son ecuaciones funcionales.  Anteriormente se mencionó que: el primer paso que se debe de dar para solucionar el problema de la unificación de los exocomportamientos  es explicar los exocomportamientos sensibles. Por lo tanto, a continuación se va describir más formalmente en qué consisten las funciones que pueden ser solución de la ecuación de los exocomportamientos sensibles. Las funciones $f_{s}^{S}$  se definen de la siguiente manera:
\[
f^{S}_{s}: L(\mathcal{E})\times P_{1}\times \cdots\times P_{(m-n)}\times \cdots\times P_{m}\rightarrow
( A_{0}^{s}\times\cdots\times A_{p}^{s} )\times (P_{(m-n)+1}\times \cdots\times P_{m})
\]

donde cada $A_{i}^{s}$ es un conjunto de representaciones de los actos con los que puede actuar el sistema-s sobre la propiedad $i$ del universo. El conjunto   $L(\mathcal{E})$ es un conjunto de fórmulas que permiten representar estados del universo. Los conjuntos $P_{(m-n)+1},..., P_{m}$  son siempre parámetros de salida que retroalimentan la función $f^{S}_{s}$. Las funciones  $f_{s}^{S}$  pueden ser visualizadas con un enfoque de caja negra como aparece  en la figura \ref{fig:MC}.

\begin{center}
\begin{figure}
  \centering
    \includegraphics{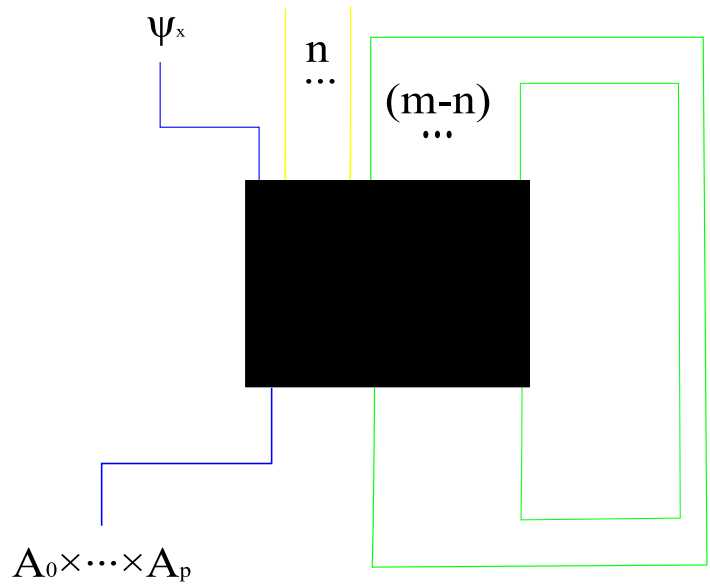}
  \caption{En la figura se observa la función $f^{S}_{s}$ como una caja negra. Los parámetros en azul siempre existen en una función $f^{S}_{s}$, y como muestra la figura estos son el parámetro de entrada con la representación del estado realidad del universo y la representación de acciones que se genera. En anaranjado son valores que son independientes de la historia del mecanismo. Finalmente, en verde los parámetros de entrada que son retroalimentados por parámetros de salida, permitiendo que el sistema sea función de la historia del propio sistema.}
  \label{fig:MC}
\end{figure}
\end{center}

\subsection{Algoritmos y Fasa }
Para terminar esta sección se va a contestar al lector sobre unas cuestiones que le pueden haber surgido: ¿Cuál es la diferencia entre el concepto de algoritmo y fasa ? ¿Todo algoritmo es la descripción de un valor de fasa? Si el lector se ha formulado estas preguntas se ha dado cuenta que la fasa es computación; pero ahora debe darse cuenta que la computación no es el único ingrediente de la fasa.  La diferencia entre el concepto de algoritmo y fasa consiste en que la fasa es la computación de unos datos concretos y relativos a la estructura del sistema, mientras que un algoritmo no pide ningún requisito a los datos que computa. En la TGE, sólo si un sistema efectúa computación sobre una representación del estado del universo se habla de que el sistema posee fasa sensible. La fasa sensible es una combinación indivisible de computación y representación. En caso de que los datos que usa para computar un sistema no fueran una representación del estado del universo sucedería que el exocomportamiento que realizaría el sistema sería posicional  o aleatorio, dependiendo de cuál sea el modo en que se generen los datos que entran a computarse en el sistema; pero no sería exocomportamiento sensible. Se pondrá a continuación un ejemplo que aclare esta última cuestión:

\textit{Imagínese que se tiene el famoso sistema experto Prospector, que se dedica a decidir cuándo se encontrará el suficiente mineral de un tipo como para que sea rentable la operación de extraer mineral. Se monta una máquina totalmente autónoma capaz de hacer extracciones de material por si misma donde el subsistema que la dirige es Prospector. En este escenario se tienen tres situaciones posibles para suministrar datos a Prospector: a) se dota a la máquina de mecanismos para hacer análisis geológicos del ambiente autónomamente; b) la máquina genera datos geológicos por medio de variables aleatorias independientes del ambiente; c) la máquina genera una secuencia de datos geológicos continuamente.}
\newline
\newline

Evidentemente sólo en el caso a) el sistema posee fasa sensible, ya que en los otros dos casos el exocomportamiento del sistema no tendrá ningún sentido respecto al objetivo de extraer mineral de una manera rentable. El lector debe de notar que el sistema contiene en todos los casos el mismo algoritmo de computación, lo que diferencia el caso a) del resto de casos es que los datos son una representación del entorno. De modo que el algoritmo no es la descripción total del valor de la fasa, es necesario el ingrediente de la representación. Computación y representación son dos ingredientes igual de importantes e indisolubles para que un sistema tenga fasa sensible, ya que ninguno de ellos por si sólo es suficiente para dotar al sistema de fasa sensible.

\section{Teoría Cognitiva de Condiciones de Verdad}

En la sección anterior se explicó que en la TGE se usa el enfoque científico; observando la naturaleza con los conceptos de propiedades y fenómenos. Pero la precisión que consigue, por ejemplo, la física describiendo la naturaleza surge de las descripciones matemáticas. Así, un elemento fundamental de una teoría científica es la teoría matemática que describe formalmente el fenómeno, la propiedad, y cómo están relacionados entre ellos. Anteriormente  también se ha mencionado que el desarrollo de la TGE pasa en primer lugar por desarrollar  una descripción formal para los exocomportamientos sensibles. A la parte de la TGE que se encarga de describir formalmente los exocomportamientos sensibles se la denomina teoría cognitiva de condiciones de verdad(TCCV).

\subsection{Marco Matemático para los Exocomportamientos Sensibles }
El primer paso para desarrollar una teoría formal sobre los exocomportamientos sensibles se corresponde con la elección de un marco matemático que permita describir formalmente los exocomportamientos sensibles, sus causas y su relación. Por ejemplo, la física utiliza: la geometría para describir el espacio, el cálculo vectorial para la mecánica clásica, y las ecuaciones diferenciales para la termodinámica. El marco matemático que se usa en la TCCV para desarrollar los postulados presentados en la sección anterior es el de la teoría de modelos\cite{Manzano1991}. A menos que el lector provenga del área de lógica-matemática posiblemente no conozca este marco matemático. La teoría de modelos fue desarrollada en la década de los 50's del siglo XX y su principal impulsor fue el importante matemático-lógico Alfred Tarski. La palabra modelo puede llevar a interpretaciones incorrectas de cuál es el marco matemático, ya que en informática la palabra se encuentra ligada al área de la simulación con el significado de una representación. También en el caso de la física es este su significado, como cuando se habla del modelo del átomo o  del modelo del universo. Pero el significado de la palabra modelo en la teoría de modelos difiere sustancialmente del de las anteriores áreas científicas. El sentido de la palabra modelo en la teoría de modelos es el mismo que en pintura o fotografía, el modelo es el objeto que va a ser dibujado o fotografiado, y en el caso de la teoría de modelos el objeto que va a ser representado. La teoría de modelos estudia las relaciones entre estructuras matemáticas y lenguajes formales. Lo que une los elementos de esos dos tipos de objetos es la semántica, mediante la noción de verdad. Se suele dar como ejemplo lo siguiente\cite{Feferman1993}:

\begin{center}
\textit{``La nieve es blanca''\- es verdad si, y sólo si, la nieve es blanca.}
\end{center}

La semántica conecta la frase ``La nieve es blanca'' con el hecho físico de que la nieve es blanca. El hecho físico de que la nieve es blanca se lo denomina condición de verdad.  La teoría de modelos diferencia tres tipos de objetos: modelos, lenguajes formales e interpretaciones. Los modelos son los objetos que se pueden representar, los lenguajes formales las representaciones, y las interpretaciones son los elementos que indican que parte de un modelo está representando un término, o una fórmula, de un lenguaje formal. Pero, ¿por qué este marco y no otro? Bien, para construir una teoría general formal sobre los fenómenos de los exocomportamiento sensibles,  la definición del tipo de exocomportamiento sensible exige que en el marco matemático se  diferencie perfectamente entre lo representado y la representación. La teoría de modelos cumple con esta condición, diferencia perfectamente entre lo representado y la representación. A continuación, se mostrará un ejemplo sencillo que ilustra el marco de la teoría de modelos. En la teoría de modelos, un modelo es una estructura matemática. El siguiente ejemplo usa una estructura que será familiar,  la  relación de orden.

Sea la estructura matemática  $\mathcal{O} = (\mathbf{O}, < )$ donde $\mathbf{O} = \{ 1, 2, 3\}$ y la relación de orden es $<\; =\{\; 1<2,\quad 1<3,\quad 2<3  \;\}$ . Por otro lado se define un lenguaje de primer orden con igualdad (variables, constantes, cuantificadores, símbolos auxiliares), $L(\mathcal{O})$, adecuado a la estructura. El lenguaje se compone de fórmulas y cada una expresa algo sobre la estructura matemática. Un tipo de fórmulas muy especiales son los axiomas. La razón de su importancia es que mediante los axiomas se puede caracterizar a un conjunto de muchas estructuras, llamado clase, y si de los axiomas se deriva un teorema, este se cumplirá en todas las estructuras de la clase. En el caso de la clase de los estructuras de orden sus axiomas son:

\[\begin{array}{lllc}
  -   \neg\exists x\in\mathbf{O}&   \qquad& x<x\\
  -  \neg\exists x,y\in\mathbf{O} & \qquad& x<y\quad \wedge\quad y < x \\
  - \forall \; x,\; y,\; z \in \mathbf{O}& \qquad& x<y\quad \wedge\quad y < z \Longrightarrow x < z
\end{array}
\]
Pero las fórmulas también pueden decir cosas más concreta como $1<3$ y $3<1$. Por supuesto, lo que interesa es saber si lo que se dice sobre la estructura es verdad, para ello la teoría de modelos utiliza un elemento llamado interpretación, $\mathcal{I}$, que consiste en una función que hace corresponder a elementos del lenguaje formal con elementos de la estructura matemática. De este modo se dice que
\[
\mathcal{O}\mathcal{I} \Vdash 1<3
\]
\[
\mathcal{O}\mathcal{I} \nVdash 3<1
\]
donde el símbolo $\Vdash$ significa ``es modelo'' y $\nVdash$ ``no es
modelo''.

\subsection{Formalización de las Ecuaciones}
En la demostración del teorema fundamental de los exocomportamientos se presentó una expresión que denotaba la relación entre fasa sensible y exocomportamiento sensible. El paso que se va a dar a continuación es traducir la expresión al marco matemático que se ha escogido. Durante el proceso de formalización se debe diferenciar entre la asignación a un concepto informal de un elemento formal del marco matemático y la igualdad de las ecuaciones de la igualdad entre dos elementos matemáticos. Para ello se usará el siguiente convenio de notación, la asignación se representará con un $\equiv$, y para la igualdad ecuacional se usará el símbolo clásico $=$.

\subsubsection{Descripción Formal de los Fenómenos}
En la teoría de modelos, la estructura matemática es el elemento que fija lo que es verdad, ya que en él se encuentran las condiciones de verdad, y por lo tanto, es el elemento que permite explicitar lo que ocurre en un universo. Así, el elemento que adopta el papel de describir el exocomportamiento sensible es la estructura matemática. Las estructuras matemáticas que estudia la TES tienen una función que simboliza la secuencia de actos que realizan los sistemas. A esa función se la denomina función interacción del sistema-s, y se representa por $\quad\funi_{s}$ .  Su definición formal es:

\[
\begin{array}{cccc}
  \funi_{s} & \mathbf{T} &\longrightarrow &\mathbf{A}_{0}^{s}\times\cdots\times\mathbf{A}_{p}^{s} \\
   & t & \mapsto & \langle\mathbf{a}_{o},...,\mathbf{a}_{p}\rangle
\end{array}
\]

Recuérdese que la forma de las ecuaciones que se han derivado de los postulados es:

\[
\varepsilon_{s}(t)=\rho f^{S}_{s}(\psi_{r},p_{1},...,p_{m})
\]

Así pues,
\[
\varepsilon_{s}(t)\equiv  \funi_{s}
\]

Probablemente a estas alturas el lector se esté preguntando como son más exactamente esas estructuras  matemáti\-cas (o modelos) de las ecuaciones. Lo más común es que el lector conozca estructuras ``estáticas'' como grupos, relaciones de orden o espacios vectoriales; pero evidentemente, estas ecuaciones no hablan de ese tipo de estructuras. Las estructuras de las que hablan las ecuaciones de la TES son una modificación del tipo que existen en la lógica modal, las llamadas estructuras de Kripke. Una estructura de Kripke es un conjunto de elementos, llamados mundos, y una relación, llamada relación de accesibilidad, que dicta con que mundos están relacionados cada mundo. Así, la idea que hay detrás de una estructura de la TES es usar una estructura de Kripke para que cada mundo represente un estado del universo y que la relación de accesibilidad, conecte los estados cada uno de los estados con todos los estados que pueden ser generados por los actos que se pueden realizar. Pero una estructura de la TGE también es diferente de una estructura de Kripke. En una estructura de la TGE se debe fijar un estado como inicial y la relación de accesibilidad es una terna de estado, acto y estado. Además, la  estructura posee la función interacción $\quad\funi_{s}$ , que dicta los actos que realiza cada sistema en cada momento del tiempo. A las estructuras de este tipo se la denota por $\mathcal{E}$. La figura \ref{fig:estructuraE} visualiza una estructura $\mathcal{E}$.

\begin{center}
\begin{figure}
  \centering
    \includegraphics{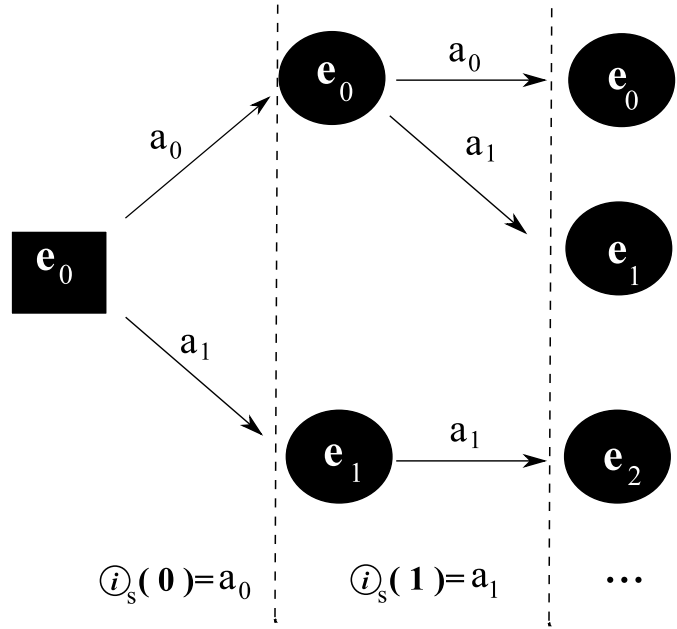}
  \caption{Visualización de una estructura de la TGE extendida en el tiempo. Los  estados del universo vienen representados por los nodos del grafo. Las flechas  representan actos. Si el universo está en un estado, y se produce un determinado acto, el universo queda en otro estado concreto. Así, en el diagrama cada una de las flechas representa un acto y  cada una conecta un estado  en que se produce un acto y el que se genera a causa de ese acto. El estado inicial en el que se encuentra el universo viene denotado por la forma cuadrada del nodo. La función interacción dicta que interacción se lleva a cabo en cada instante del tiempo. }
  \label{fig:estructuraE}
\end{figure}
\end{center}

 Alguno de los lectores podría estar pensando que las estructuras matemáti\-cas de la clase $\mathcal{E}$  usan tiempo discreto, lo cual parece que está alejado del ser humano que piensa y actúa en un tiempo continuo. Nada más alejado de la verdad. El sistema de control de los seres humanos, el sistema nervioso, opera intermitentemente. Numerosas investigaciones han llevado a comprender que el movimiento no es tan armónico como lo hace parecer los movimientos voluntarios, en realidad, la ejecución motora consiste en una serie intermitente de sacudidas cuya periodicidad está entre 8 y 12 Hz\cite{Llinas1991}.

\subsubsection{ Descripción Formal de la Fasa Sensible}

En el término derecho de la ecuación, se encuentra la función  $f_{s}^{S}$ que tiene un parámetro que es una representación del estado realidad. En teoría de modelos, son las fórmulas del lenguaje formal las que se ocupan de representar. De modo que para formalizar la representación es necesario definir un lenguaje formal  $L(\mathcal{E} )$ que permita expresar relaciones y características de los elementos de $\mathcal{E}$, y la representación queda formalizada en  el marco matemático como una fórmula  $\psi_{r}$ que habla sobre el estado realidad en el instante $t$ .

\subsubsection{Descripción Formal de la Relación entre el Exocomportamiento Sensible y la Fasa Sensible}

Una vez se ha explicado cómo se formalizan el exocomportamiento sensible y la fasa sensible de un sistema en el marco matemático elegido. El siguiente paso es especificar y formalizar $\rho$ en el marco de la teoría de modelos. El elemento $\rho$  representa la relación que existe entre el exocomportamiento sensible y la fasa sensible de un sistema. La relación debe de cumplir dos condiciones:
\begin{enumerate}
  \item Hay que transformar la salida de la función $f_{s}^{S}$ en un elemento del modelo. La razón es que la función $f_{s}^{S}$  devuelve únicamente un elemento del lenguaje, la representación de un acto; pero en el lado derecho de la ecuación los objetos pertenecen al modelo. Por lo tanto no se podría poner un igual entre ellos porque son objetos de distinta naturaleza matemática.
  \item Se debe extraer de la secuencia de parámetros devuelta por la función $f_{s}^{S}$  siempre el primero. Eso se debe a que la función $f_{s}^{S}$  puede devolver varios parámetros, por lo que se debe de extraer aquel parámetro en el que se encuentra la representación de la acción, y como por definición siempre se encontrará en la posición primera, es ese el parámetro que se debe de extraer.
\end{enumerate}

La primera condición se puede expresar formalmente simplemente realizando una interpretación de la representación del acto sobre la realidad. Recuérdese, que en cada instante de tiempo se pone la etiqueta $\mathbf{r}_{\mathbf{t}} $  al estado en el que se encuentra el universo. Así, las representaciones de los actos son interpretadas por el estado realidad.  De manera que la representación  del  acto, $a_{i}$, queda transformada en el acto, $\mathbf{a}_{i}$. Es decir:
\[
\mathbf{r}_{\mathbf{t}} (a_{i}) = \mathbf{a}_{i}
\]

De este modo, la traducción formal queda así:

\[
\rho\equiv \mathbf{r}_{\mathbf{\mathbf{t}}}()
\]
Para que la relación cumpla también la segunda condición hay que formalizar la operación de extraer de la secuencia de parámetros el primero, ya que por la definición siempre se encontrará en esa posición la secuencia de representaciones de actos. Para ese cometido se usará una función proyección  $\pi_{c}$ que seleccione la primera componente de la tupla. Por lo tanto:

\[
\rho\equiv \mathbf{r}_{\mathbf{t}}\circ\pi_{1}
\]

Recuérdese que la relación es una relación matemática entre la propiedad y el fenómeno; la naturaleza no realiza cálculos matemáticos. Cuando Newton dice: ``la fuerza de la gravedad es inversamente proporcional al cuadrado de la distancia'' no quiere decir que la naturaleza calcule el cuadrado de la distancia, sólo refleja cómo se comporta la naturaleza. Puesto que la función  siempre debe devolver al menos un parámetro que consiste en una representación de un acto, el uso de la función proyección no provoca ninguna incompatibilidad aunque la función solo devolviera un parámetro. Por lo que se establece que:

\[
\rho\equiv \mathbf{r}_{\mathbf{t}}\circ\pi_{1}
\]

\subsubsection{ Descripción Formal de  las Ecuación General del Exocomportamiento Sensible}

El siguiente paso es sintetizar todos los anteriores elementos para formalizar la ecuación entre exocomportamiento sensible y fasa sensible. Así, sustituyendo por cada elemento formal del marco de la teoría de modelos se obtiene la ecuación:

\begin{equation}
\funi_{s}(\mathbf{t})=\mathbf{r}_{\mathbf{t}}(\pi_{1} ( f^{S}_{s}(\psi_{r},p_{1},...,p_{m})))
\end{equation}

A esta  expresión matemática se la denominará ecuación general del excomportamiento sensible.
Uno debe darse cuenta que esta ecuación se corresponde en verdad con toda una familia de ecuaciones, ya que dependiendo del número de parámetros que tenga  $f_{s}^{S}$ surgen diferentes ecuaciones. En la sección 4 se entrará en profundidad sobre $f_{s}^{S}$ y $p_{1},..., p_{m}.$

Una cuestión relevante que debe de quedar bien clara referente al trabajo matemático que se realiza en la TCCV, ya que el lector podría confundirse por sus experiencias anteriores en el campo de la lógica, o las teorías matemáticas. En la TCCV se buscan estructuras matemáticas que especifiquen comportamientos sensibles para posteriormente resolver las ecuaciones que surgen con las distintas arquitecturas funcionales. Para esas estructuras matemáticas,  posiblemente se pueda dar un conjunto de axio\-mas. Esos axio\-mas dan lugar a una teoría; pero el lector no debe confundir la teoría que se derive de los axiomas con la TCCV. La teoría que se genera a partir de los axiomas únicamente establecerá teoremas sobre la clase de estructuras matemáticas que define los axiomas. Por lo tanto, no se debe de cometer el error de identificar la teoría de la clase de estructuras matemáticas que se usan en una ecuación de la TCCV con la TCCV.

Uno debe darse cuenta que esta ecuación se corresponde en verdad con toda una familia de ecuaciones, ya que dependiendo del número de pa\-rá\-metros que tenga  $f_{s}^{S}$ surgen diferentes ecuaciones. En la sección 4 se entrará en profundidad sobre $f_{s}^{S}$ y $p_{1},..., p_{m}$.

\subsection{La TCCV y Técnicas de Lógica en I.A.}
Desde el nacimiento de la I.A. en 1956 hasta los años ochenta, el camino de la lógica, liderado especialmente por John McCarthy y Marvin Minsky, fue el camino principal de investigación para cumplir el objetivo de Dartmouth. Algunos de los más importantes hitos durante ese camino fueron el descubrimiento de J. Alan Robinson en 1965 del principio de resolución y unificación, la sugerencia en 1974 Robert Kowalski de representar las expresiones lógicas como cláusulas de Horn, la aparición de la programación lógica, y la creación de programas como STRIPS. Estos hitos hicieron que la lógica  dominara el campo de la inteligencia artificial durante casi 20 años. Pero a principio de los años 80 los progresos habían cesado y existía un estancamiento, a ello se unió el fracaso del proyecto de 5ª generación del gobierno de Japón y la aparición de determinados problemas frente a los cuales la lógica clásica se mostraba ineficiente. Ese conjunto de hechos, junto al desarrollo de nuevos métodos que si se mostraban eficientes para los problemas que la lógica clásica (como probabilistas, redes neuronales, redes bayesianas,..) condujeron a que la lógica perdiera su posición de dominio a pesar de los éxitos iniciales.

Alguien podría objetar a la TCCV que a pesar de que se llega a plantear el marco de la teoría de modelos a través de una teoría científica sobre la naturaleza, parece que vuelve a proponer a la lógica como línea fundamental de investigación; lo cual sería estar en un punto en el que ya se estuvo en el pasado y que fue necesario abandonar. Nada más lejos de la verdad, se debe de notar que la TCCV no pone ningún tipo de restricción a cual es el mecanismo que asocia a los elementos del dominio de $f_{s}^{S}$ con sus imágenes. Por lo tanto, en ningún momento se dice que tenga que ser un cálculo lógico, como un cálculo de  situación \cite{McCarthy69}\cite{McCarthy20}. Por lo tanto, $f_{s}^{S}$ puede ser definida por un método probabilista, fuzzy, una red neuronal,..., etc.

La idea de usar un método de representación y otro de computación es algo que está siendo investigado en una línea, que se está desarrollando en los últimos años, llamada integración neuro-simbólica\cite{Hitzler2004},  que busca integrar programas lógicos y redes neuronales.

\subsection{ Descomposiciones y Arquitecturas Funcionales}
\begin{definicion}[ Descomposición funcional de $f_{s}^{S}$] Se denomina descomposición funcional de $f_{s}^{S}$ a una expresión  formada con funciones mediante las operaciones de combinación y composición, y que es equivalente a la función $f_{s}^{S}$  en la ecuación general del comportamiento sensible.
\end{definicion}

\begin{definicion}[Arquitectura funcional de $f_{s}^{S}$] Se denomina arquitectura funcional de $f_{s}^{S}$ a una expresión con funciones formada mediante las operaciones de combinación y composición, y que puede sustituir a una función $f_{s}^{S}$  en la ecuación general del comportamiento sensible.
\end{definicion}
\begin{definicion}[Unidad funcional de $f_{s}^{S}$] Se denomina unidad funcional a cada una de las funciones que forma parte de una arquitectura funcional de $f_{s}^{S}$.
\end{definicion}
Debe de notarse que la diferencia entre descomposición funcional y arquitectura funcional consiste en que la descomposición funcional admite todas las funciones que admite $f_{s}^{S}$ y la arquitectura funcional sólo un subconjunto de ellas. La relación entre descomposiciones y arquitecturas funcionales es que fijando parámetros en una descomposición obtendremos distintas arquitecturas. El objetivo de los siguientes apartados es mostrar un método de descomposición  para la función $f_{s}^{S}$ y mostrar algunas arquitecturas funcionales posibles.

\subsubsection{Descomposición de $f_{s}^{S}$}
A continuación se va a presentar una descomposición de la función  $f_{s}^{S}$, cuyos parámetros permitirán definir arquitecturas funcionales. La  descomposición consiste en tres funciones de la siguiente manera:

\[
f_{s}^{S}= (\pi_{c}\circ g^{d}_{s})\times h_{s}
\]

\[
g^{d}_{s}: L(\mathcal{E})\times P_{1}\times \cdots\times P_{(m-n)}\times \cdots\times P_{m}\rightarrow
\bigcup^{j=d}_{j=1}({\it A_{0}^{u}\times\cdots\times A_{p}^{u}
})^{j}
\]
\[
h_{s}: L(\mathcal{E})\times P_{1}\times \cdots\times P_{(m-n)}\times \cdots\times P_{m}\rightarrow
P_{(m-n)+1}\times \cdots\times P_{m}
\]

En las funciones $g^{d}_{s}$ se llama \emph{profundidad máxima} al valor que tenga $d$.
Esta descomposición de la función permite $f_{s}^{S}$ permite, fijando parámetros en las funciones que forman la configuración, definir subconjuntos de valores que puede tomar la fasa sensible de un sistema. Nótese que dada una función  $g^{d}_{s}$ dependiendo de la función proyección $\pi_{c}$ se da lugar a distintas  $f_{s}^{S}$ . A su vez, se puede ver que el conjunto de funciones que pueden sustituir a $f_{s}^{S}$ se divide en subconjuntos en función del valor de $c$. De esta manera, surgen las dos siguientes definiciones:

 \begin{definicion}[ Función $f_{s}^{S}$ asociada a una función $g^{d}_{s}$ vía $c$] Se dirá que una función $f_{s}^{S}$  está asociada a una función $g^{d}_{s}$ vía $c$ si se cumple que:

\[
f_{s}^{S}= (\pi_{c} \circ g^{d}_{s})\times h_{s}
\]
\end{definicion}

 \begin{definicion}[ Conjunto de funciones $f_{s}^{S}$  vía $c$] Se define el conjunto de funciones $f_{s}^{S}$  vía $c$, denotado por $\mathbb{F}^{c}_{s}$ ,  al conjunto de funciones  siguiente:

\[
\mathbb{F}^{c}_{s} = \{f_{s}^{S} : f_{s}^{S}= (\pi_{c} \circ g^{d}_{s})\times h_{s} \}
\]

\end{definicion}

En una arquitectura funcional podrían existir una redundancia de unidades funcionales debido a que hubiera unidades funcionales biyectivas cuya salida fuera la entrada de otra unidad funcional que fueran la inversa de la unidad funcional que ha generado la entrada, de manera que no causarían ninguna alteración pero si un gasto al sistema.

\begin{definicion}[Arquitectura funcional redundante]
Se dice que una arquitectura funcional es redundante cuando dentro de la arquitectura funcional existen tres unidades funcionales $f, f ^{-1}$ y $g$ , entre las que se da la siguiente configuración:
\[
g( f ^{-1}(f(a)  ) )
\]
donde la unidad funcional $f$,  es biyectiva,  la unidad funcional $f ^{-1}$ es la inversa de $f$ y $g$  puede ser cualquier unidad funcional incluso la misma $f$.
\end{definicion}

Se considerará que todas las arquitecturas funcionales que se estudian son no redundantes.

\subsubsection{ Conjunto de Exocomportamientos Sensibles Ordenados}

De todos los posibles conjuntos $\mathbb{F}^{c}_{s}$   interesa especialmente  $\mathbb{F}^{1}_{s}$ , al que se denominará conjunto de exocomportamientos sensibles ordenados. El conjunto  $\mathbb{F}^{1}_{s}$ contiene subconjuntos muy interesantes de exocomportamientos. La razón de fijar $c = 1$  proviene de la función $g^{d}_{s}$.  La función $g^{d}_{s}$ genera una secuencia de representaciones de actos pero desconoce lo que hará la función proyección.  La función proyección debe de extraer únicamente uno de las representaciones de actos de la secuencia que genera $g^{d}_{s}$. Por lo tanto, un conjunto de arquitecturas muy interesantes son aquellas en las que la función proyección mantiene el orden de la secuencia generada. Aunque puede parecer algo no muy significativo, con el siguiente ejemplo se apreciará la importancia del orden de la secuencia.
\newline
\textit{ Imagínese que existe una persona que no realizara los actos en el mismo orden que los planea. Así, cuando esta persona se propone beber agua de una botella ocurre lo siguiente. Inicialmente planea la siguiente secuencia de actos: (1) coger la botella (2) quitar el tapón de la botella (3) inclinar la botella sobre su boca abierta (4) poner el tapón en la botella.  A continuación, la persona realiza los mismos actos pero en el siguiente orden (1) coger la botella (2) quitar el tapón de la botella (3) poner el tapón en la botella (4) inclinar la botella sobre su boca abierta.}
\newline

En este ejemplo, la persona no ha llegado a beber agua, que era el estado al que conducía la secuencia generada, de manera que, aunque tiene capacidad para generar la secuencia adecuada, su exocomportamiento no es efectivo para lograr un estado concreto. La razón de la ineficiencia es que los actos no son conmutativos, de manera que si los actos no mantienen el orden, entonces, al igual que ocurre cuando la estabilidad básica es baja o la inestabilidad es alta, el exocomportamiento es ineficiente para cumplir las restricciones de exoactividad. A cada exocomportamiento sensible que mantiene el orden se le denominará \emph{exocomportamiento sensible ordenado}.

\begin{definicion}[Exocomportamiento sensible ordenado] Un sistema lleva a cabo un exocomportamiento sensible ordenado si el acto que realiza en cada momento del tiempo coincide con la representación del acto que está en la primera posición de la secuencia que ha generado la función $g^{d}_{s}$.
\end{definicion}

\begin{definicion}[ $f_{s}^{S}$ -ordenada] Una función  $f_{s}^{S}$   que genera un  exocomportamiento ordenado se la denominará función  $f_{s}^{S}$-ordenada
\end{definicion}

\begin{definicion}[Arquitectura funcional sensible ordenada] A una arquitectura que sea equivalente a la arquitectura  $ (\pi_{1}\circ g^{d}_{s})\times h_{s}$ se la denominará arquitectura funcional ordenada.
\end{definicion}

\subsubsection{Arquitecturas Funcional Sensible I}
La primera arquitectura que se va a presentar es la más simple de todas, se denomina arquitectura funcional sensible I(AFS-I). La AFS-I se obtiene fijando en la descomposición, que fue presentada, los valores de los parámetros $c$ y $d$  a $1$ y se elimina la función $h_{s}$. De esa manera la expresión de la arquitectura funcional sensible queda como:

\[
f_{s}^{S}= \pi_{1}\circ g^{1}_{s}
\]

La función  $g^{1}_{s}$ tomará como valor una función a la que se denomina función predicción de tipo-$\alpha$ .

\begin{definicion}[ Función predicción de tipo $\alpha$]
La función predicción de tipo $\alpha$, que se denotará por  $\vdash_{\alpha}$, se define de la siguiente manera:
\[
\begin{array}{cccc}
  \vdash_{\alpha}^{s}: & L(\mathcal{E}) &\longrightarrow & A_{0}^{s}\times\cdots\times A_{p}^{s}\cup \{ \emptyset \} \\
   & \psi_{x} & \mapsto & \langle a_{o},...,a_{p}\rangle
\end{array}
\]
donde $\psi_{x}$ es una representación de un estado del universo y $\langle a_{o},...,a_{p}\rangle$ un acto del sistema-s.
\end{definicion}

Así, la función de predicción de tipo-$\alpha$  devuelve la representación de los actos con los que el sistema reacciona a la representación del estado  $\psi_{x}$. La función de predicción puede no devolver ningún acto, en ese caso la función interacción se iguala con el acto neutro, representado por $\langle e_{0},...e_{p} \rangle$ .

\[
f_{s}^{S}(\psi_{r_{t}}) =\left \{
  \begin{array}{ll}
   \pi_{1}\big(\;\vdash_{\alpha}(\psi_{r_{t}})\big) , &\vdash_{\alpha}(\psi_{r_{t}})\neq \emptyset
   \\ \\
      \langle e_{0},...e_{p} \rangle , & \vdash_{\alpha}(\psi_{r_{t}})=\emptyset
  \end{array}
\right.
\]

De ahora en adelante, para facilitar las notaciones se omitirá el caso se la secuencia vacía.
\newline
La ecuación general del exocomportamiento sensible que aparece con esta arquitectura funcional es la ecuación \ref{ecua:AFS-I}.

\begin{equation}
\funi_{s}(\mathbf{t})=\mathbf{r_{\mathbf{t}}}(\pi^{2}_{1}(\vdash_{\alpha}(\psi)))
\label{ecua:AFS-I}
\end{equation}
Para facilitar la notación de la escritura de las ecuaciones, se ha usado un convenio para los casos en los que hay varias funciones proyección consecutivamente que seleccionan el mismo componente de las respectivas tuplas. El símbolo $\pi$  con su subíndice $c$ servirán para denotar la componente que selecciona de la tupla la función proyección. Además,  se usará un segundo superíndice para designar el número de funciones proyección que hay consecutivamente colocadas. Es decir:

\[
\pi^{2}_{c} = \pi_{c}\circ \pi_{c}
\]

Esta ecuación describe sistemas que reaccionan al estado del entorno, como podrían ser ciertas plantas.

\subsubsection{Arquitecturas Funcional Sensible IIA }
La arquitectura funcional que se va presentar a continuación, denominada arquitectura funcional sensible IIA(AFS-IIA),  consiste en una variación de AFS-I. La variación consiste en añadir un parámetro  de entrada independiente de las salidas de la función. La función  $g^{d}_{s}$ tomará como valor la función a la que se denomina función predicción de tipo $\alpha,\beta$ ,

\begin{definicion}[ Función predicción de tipo-$\alpha,\beta$]
La función predicción de tipo-$\alpha,\beta$, que se denotará por  $\vdash^{s}_{\alpha,\beta}$, se define de la siguiente manera:

\[
\begin{array}{cccc}
  \vdash^{s}_{\alpha,\beta}: & L(\mathcal{E})\times L(\mathcal{E}) &\longrightarrow& \bigcup^{j=d}_{j=1}( A_{0}^{s}\times\cdots\times A_{p}^{s})^{j} \cup \{ \emptyset \} \\
   & (\psi_{x},\varphi) & \mapsto & \langle \overline{a}_{1},...,\overline{a}\rangle_{d'}
\end{array}
\]
donde $d'\leq d$ y $\overline{a}_{i}\in A_{0}^{s}\times\cdots\times A_{p}^{s}$
\end{definicion}

Puesto que todos los elementos que intervienen en una ecuación se refieren siempre al  mismo sistema, a partir de ahora se eliminará de la notación el índice que indica  el sistema en la parte derecha de las ecuaciones. La ecuación que resulta con esta arquitectura es:

Puesto que todos los elementos que intervienen en una ecuación se refieren siempre al  mismo sistema, a partir de ahora se eliminará de la notación el índice que indica  el sistema en la parte derecha de las ecuaciones. La ecuación que resulta con esta arquitectura es:
\begin{equation}
\funi_{s}(\mathbf{t})=\mathbf{r_{\mathbf{t}}}(\pi^{2}_{1}(\vdash_{\alpha,\beta}(\psi,\varphi)))
\label{ecua:AF-IIA}
\end{equation}

Cada función predicción de tipo $\alpha,\beta$ genera una secuencia de representaciones de actos que es función de sus parámetros de entrada  $\psi$ y $\varphi$. De todas las posibles funciones predicción de tipo $\alpha,\beta$ que se pueden dar, existe un subconjunto que interesa especialmente, las funciones predicción orientadas.

\begin{definicion}[Función predicción orientada] Una función predicción orientada es una función predicción en la que se cumple que la secuencia de representaciones de acciones que genera a partir de $\psi$ y $\varphi$  hace que  partiendo de $\psi$  se  genere $\varphi$.
\end{definicion}

Ahora bien, de todas las posibles arquitecturas que tienen la forma anterior interesan aquellas que pertenezcan al conjunto de comportamientos sensible ordenados. Esto es, aquellas que toman para $c$ el valor $1$ . La función $g_{s}^{d}$ genera una secuencia de representaciones de actos, y el fenómeno es una secuencia de actos, por lo tanto para que la secuencia del comportamiento sea eficiente debe cumplirse que mantenga el orden de la secuencia generada. Así,  fijando el parámetro $c =1$ se define la AFS-IIA, cuya expresión es:

\[
f_{s}^{S}= \pi_{1}\circ((\pi_{1}\circ\vdash_{\alpha,\beta})\times h_{s})
\]
Si se sustituye  la función $f_{s}^{S}$  por la  AFS-IIA en la ecuación general del comportamiento sensible surge la ecuación \ref{ecua:AF-IIA}:

\begin{equation}
\funi_{s}(\mathbf{t})=\mathbf{r_{t}}(\pi^{2}_{1}(\vdash_{\alpha,\beta}(\psi,\varphi)\times h_{s}(\psi,\varphi)))
\label{ecua:AF-IIA}
\end{equation}

Esta ecuación describe sistemas que intentan lograr objetivos, como podría ser el no tener hambre.
Hay que darse cuenta que no todos los exocomportamientos del conjunto de exocomportamientos sensibles ordenados y orientados permiten que el sistema persista en el universo.

\subsubsection{Arquitecturas Funcional Sensible IIB }
Esta arquitectura es una variación de la arquitectura AFS-I añadiendo un parámetro de retroalimentación sirve como mecanismo de memoria. La función  $g^{d}_{s}$ tomará como valor la función a la que se denomina función predicción de tipo  $\alpha,\gamma$.

\begin{definicion}[ Función predicción de tipo-$\alpha,\gamma$]
La función predicción de tipo-$\alpha,\gamma$, que se denotará por  $\vdash^{s}_{\alpha,\gamma}$, se define de la siguiente manera:
\[
\begin{array}{cccc}
  \vdash_{\alpha,\gamma}: & L(\mathcal{E})\times L(\mathcal{E}) &\longrightarrow& \bigcup^{j=d}_{j=1}( A_{0}^{s}\times\cdots\times A_{p}^{s})^{j} \cup \{ \emptyset \} \\
   & (\psi_{x},\phi) & \mapsto & \langle \overline{a}_{1},...,\overline{a}\rangle_{d'}
\end{array}
\]
donde $d'\leq d$ y $\overline{a}_{i}\in A_{0}^{s}\times\cdots\times A_{p}^{s}$
\end{definicion}

\begin{definicion}[ Arquitectura funcional sensible IIB]

La arquitectura funcional sensible IIB se define como:

\[
f_{s}^{S}= (\pi_{c}(g^{d}_{s}()))\times h_{s}()
\]

donde

\[
g^{d}_{s} =\vdash_{\alpha,\gamma}
\]

\[
\begin{array}{cccc}
  h_{s}: & L(\mathcal{E})\times L(\mathcal{E}) &\longrightarrow &L(\mathcal{E}) \\
   & (\psi_{x},\phi) & \mapsto & \phi'
\end{array}
\]
\end{definicion}

Sustituyéndose $f^{S}_{s}$ en la ecuación general del exocomportamiento sensible por la AFS-IIA resulta la ecuación   \ref{ecua:AFS-IIB}.

\begin{equation}
\funi_{s}(\mathbf{t})=\mathbf{r_{t}}(\pi^{2}_{1}(\vdash_{\alpha,\gamma}(\psi,\phi)))
\label{ecua:AFS-IIB}
\end{equation}

\subsubsection{ Arquitecturas Funcional Sensible IIIA}

La arquitectura funcional que se presenta ahora, denominada arquitectura funcional sensible IIIA, es interesante porque permite  ver como dentro de las ecuaciones de la TCCV, caben más conceptos que el de predicción. En este caso, se trata de una arquitectura que representa la capacidad de aprender. Dentro de la arquitectura hay una función, denominada función aprendizaje,  que en función de la información sobre el pasado decide asignar una nueva función de predicción para generar la secuencia de actos que realizará  el sistema.

\begin{definicion}[ Función Aprendizaje de Predicción ]

La función aprendizaje de predicción, que se denotará por $ \circledast$, se define de la siguiente manera:

\[
\begin{array}{llll}
 \circledast :&\bigcup_{j=0}^{\infty}(L(\mathcal{E})\times\mathbb{P}_{d}^{\alpha,\beta} )^{j} & \longrightarrow &\mathbb{P}_{d}^{\alpha,\beta}\\
              &   \overline{m}&\mapsto&  \vdash_{s}^{i}
 \end{array}
\]
donde $\mathbb{P}^{\alpha,\beta}_{d}= \{ \vdash_{s}^{1},..., \vdash_{s}^{\omega} \}
 \qquad
\overline{m} \in \bigcup_{j=0}^{\infty}(L(\mathcal{E})\times\mathbb{P}_{d}^{\alpha,\beta} )^{j}$
\end{definicion}

\begin{definicion}[Arquitectura funcional sensible IIIA]
La arquitectura funcional sensible IIIA se corresponde con las expresión de funciones siguiente:
\[
f_{s}^{S}= (\pi_{1}(g^{d}_{s}()))\times h_{s}()
\]

donde
\[
\begin{array}{llll}
g^{d}_{s}  :&L(\mathcal{E})\times L(\mathcal{E})\times \bigcup_{j=0}^{\infty}(L(\mathcal{E})\times\mathbb{P}_{d}^{\alpha,\beta} )^{j}  & \longrightarrow & \bigcup^{j=d}_{j=1}( A_{0}^{s}\times\cdots\times A_{p}^{s})^{j} \cup \{ \emptyset \}  \\
\\
&  (\quad \phi,\;\varphi_{s},\;\overline{m} \quad)&\mapsto & (\circledast(\overline{m}))(\phi,\;\varphi_{s})
 \end{array}
\]

\[
\begin{array}{llll}
h_{s}  :&L(\mathcal{E})\times L(\mathcal{E})\times \bigcup_{j=0}^{\infty}( L(\mathcal{E})\times\mathbb{P}_{d}^{\alpha,\beta} )^{j}  & \longrightarrow &\bigcup_{j=0}^{\infty}(L(\mathcal{E})\times\mathbb{P}_{d}^{\alpha,\beta} )^{j} \\
              &  (\quad \phi,\;\varphi_{s},\;\overline{m} \quad)&\mapsto&  \overline{m}'
 \end{array}
\]
\end{definicion}

 La ecuación que resulta al sustituir la función $f_{s}^{S}$ por la AFS-IIIA en la ecuación general de los exocomportamientos sensibles y definiendo que todos los exocomportamientos sean sensibles ordenados es la ecuación \ref{ecua:AFS-IIIA}.

\begin{equation}
\funi_{s}(\mathbf{t})=\mathbf{r_{t}}(\pi_{1}(
\pi_{1}(g^{d}_{s}(\phi,\;\varphi_{s},\;\overline{m})))\times h(\phi,\;\varphi_{s},\;\overline{m}))
\label{ecua:AFS-IIIA}
\end{equation}

Como ya se ha mencionado estas ecuaciones no son ecuaciones numéricas, sino ecuaciones funcionales. La incógnita de las ecuaciones en el caso de los tipos I y II-A es la función predicción que la resuelve. Dado un exocomportamiento sensible la solución de la ecuación del comportamiento sensible puede no ser única. Evidentemente una de las cuestiones más interesante es ver algunas soluciones de estas ecuaciones; pero por ser este un artículo de introducción a la TES no se puede profundizar en esta cuestión\cite{Miguel2006}. Sin embargo hay cosas que deben de notarse sobre las ecuaciones. Una de esas cuestiones es que no se conoce un método mecánico para hallar las funciones que son solución de las ecuaciones. Otra cuestión que debe de notarse es que  las ecuaciones cubren tanto los sistemas con complejidad biológica como los de complejidad no biológica, ya que en ellas no hay nada que diferencie entre unos y otros tipos de sistemas.

\subsection{Parámetros de la Fasa Sensible}

Como fue mencionado en la sección 2 existen ciertos parámetros que permiten conocer si el valor de la fasa sensible de un sistema cumple las restricciones de existencia del universo del sistema. Esos parámetros son la estabilidad básica, la inestabilidad y la estabilidad total. En este apartado se estudiara y explicará la definición de estos conceptos para la arquitectura II-A, por la tanto el estudio se centrará en las funciones predicción de tipo-$\alpha,\beta$ orientadas.

\subsubsection{Estabilidad}

La estabilidad básica es una medida que cuantifica cuantos estados que cumplen las restricciones de existencia y cuantos  estados que no cumplen las restricciones de existencia conoce un sistema en base a su función $f_{s}^{S}$ . La estabilidad básica de un sistema es un número real entre 0 y 1. En esta apartado se definirá la estabilidad de una función predicción de tipo-$\alpha,\beta$. Pero antes de definir formalmente como se calcula la estabilidad básica  es necesario realizar otras definiciones.

\begin{definicion}[Conjunto Estados]
   El conjunto estados se trata de un conjunto  cuyos elementos son los estados del universo. Se representa por $E$, donde $\mathbf{E} = \{ \mathbf{e}_{1} ,...\}$
\end{definicion}

\begin{definicion}[Conjunto Estados Positivos]
El conjunto de estados positivos es el conjunto cuyos elementos son los estados del universo en los que aumenta la capacidad del sistema para persistir en estados donde el sistema no cumple las restricciones de existencia del. Se representa por $\mathbf{E}^{+} = \{ \mathbf{e}_{y_{1}} ,\mathbf{e}_{y_{2}},...\}$, donde $\mathbf{E}^{+}$ es un subconjunto de $\mathbf{E}$.
\end{definicion}

\begin{definicion}[Conjunto Estados Neutros]
El conjunto de estados  neutros es el conjunto cuyos elementos son los estados del universo que son neutros para el sistema que se estudia. Se representa por $\mathbf{E}^{\simeq} = \{ \mathbf{e}_{u_{1}} ,\mathbf{e}_{u_{2}},...\}$, donde $\mathbf{E}^{\simeq}$ es un subconjunto de $\mathbf{E}$.
\end{definicion}

\begin{definicion}[Conjunto Estados Negativos]
El conjunto de estados  negativos es el conjunto cuyos elementos son estados del universo donde el sistema que se estudia no cumple las restricciones de existencia del universo. Se representa por $\mathbf{E}^{-} = \{ \mathbf{e}_{z_{1}} ,\mathbf{e}_{z_{2}},...\}$, donde $\mathbf{E}^{-}$ es un subconjunto de $\mathbf{E}$.
\end{definicion}

\begin{definicion}[Conjunto Representaciones de Estados]El conjunto representaciones de estados es el conjunto cuyos elementos son representaciones de estados del universo del sistema que se estudia. Se representa por  $R$, donde $R = {  \psi_{1} ,...}$ y $ R \subset L(\mathcal{E})$.
\end{definicion}

\begin{definicion}[]
La función representación es una función que asigna a cada estado una fórmula de $R$. Se denota por $r$ y se define
como:

\[
\begin{array}{cccc}
  r & \mathbf{E}   & \rightarrow  & $R$ \\
   & \mathbf{e}_{x}& \mapsto  &  \psi_{x}
\end{array}
\]

No se debe confundir el estado realidad, nótese que la función representación no tiene el subíndice $t$.
\end{definicion}

\begin{definicion}[Conjunto Objetivos]
El conjunto objetivos es el conjunto cuyos elementos son representaciones de estados del universo  a los cuales conducen las secuencias de actos del sistema que se estudia. Se representa por $O$, donde $O = \{\psi_{y} ,...\}$ y  es un subconjunto de $R$.
\end{definicion}

\begin{definicion}[Conjunto Objetivos Positivos]
El conjunto objetivos positivos es el conjunto cuyos elementos son representaciones de estados del universo que pertenecen al conjunto de estados positivos. Se representa por $O^{+}$, donde
\[
O^{+} = \{  y: r^{-1}( \psi_{y}) \in  E^{+}  \}
\]
y  es un subconjunto de $R$.
\end{definicion}

\begin{definicion}[Conjunto Estados de Partida hacia el estado-$i$]
El conjunto estados de partida hacia el estado-$i$ es el conjunto cuyos elementos son estados del universo desde los cuales el sistema es capaz de llegar al estado-$i$. Su definición matemática es la siguiente:
\[
P_{i}= \{ r^{-1}(\psi_{x}): \vdash_{\alpha, \beta}(\psi_{x},\varphi_{i})  \}
\]

\end{definicion}

\begin{definicion}[Conjunto Estados Negativos con Alejamiento hacia el estado-$i$]
El conjunto estados negativos con alejamiento hacia el estado-$i$ es el conjunto de estados que pertenecen al conjunto de estados negativos para el sistema desde los cuales el sistema lleva a cabo una secuencia de acciones para alejarse de él. Su definición matemática es la siguiente:
\[
A_{i}^{-}= \{ e_{x}: e_{x}\in \mathbf{E^{-}} \wedge e_{x} = r^{-1}(\psi_{x})\wedge\vdash_{\alpha,\beta}(\psi_{x},\varphi_{i})\neq\emptyset
\}
\]
donde $\psi_{x} \in O^{+}\qquad \varphi \in O^{\simeq} $
\end{definicion}

Una vez realizadas las anteriores definiciones, ya se puede definir formalmente  estabilidad básica que posee la función predicción de tipo $\alpha,\beta$  de un sistema.

\begin{definicion}[Estabilidad Básica]
Se define estabilidad básica de una función predicción, y se denota por $\textrm{Estabilidad}_{b}$, a

\[
\textrm{Estabilidad}_{b}(\vdash_{\alpha,\beta})=\frac{\sum_{i=1}^{|O^{+}|} \frac{|P_{i}|}{|E|}
}{|O^{+}|} + \frac{\sum^{|E^{\simeq}|}_{j=1} A^{-}_{j} }{|E|}
\]
\end{definicion}
\begin{ejemplo}
Dado un universo con $5$ estados distintos $\mathbf{E}=\{ \mathbf{e}_{1}, \mathbf{e}_{2}, \mathbf{e}_{3}, \mathbf{e}_{4}, \mathbf{e}_{5} \}$ se quiere estudiar la estabilidad básica de un sistema de ese universo que posee representación para los cinco estados $R =\{  \psi_{1},  \psi_{2},  \psi_{3}, \psi_{ 4},  \psi_{5}\}$. De los cinco estados, uno es positivo para el sistema $\mathbf{E}^{+}=\{ \mathbf{e}_{1} \}$, dos son neutros $\mathbf{E}^{\simeq} =\{ \mathbf{e}_{2} , \mathbf{e}_{3}\}$ y los otros dos son negativos para el sistema $\mathbf{E}^{-}=\{ e_{4}, e_{5} \}$. La descripción del estado positivo es la única descripción que se encuentra en el conjunto de objetivos del sistema $O =\{ \psi_{1}\}$. El sistema tiene una función predicción definida de manera que devuelve representaciones de secuencias de acciones desde las representaciones de los dos estados neutros a la representación del objetivo. Es decir, la secuencia que devuelve es diferente de la secuencia vacía

\[
\vdash (\psi_{2}, \psi_{1} )\neq \emptyset
\]
\[
\vdash (\psi_{3}, \psi_{1} )\neq \emptyset
\]

También devuelve representación de secuencia de acciones desde la representación del estado positivo a la representación del estado objetivo (que en este caso es la misma representación)
\[
\vdash (\psi_{1}, \psi_{1} )\neq \emptyset
\]
 y por último la función devuelve una representación de secuencia de acciones desde la representación de uno de los estados negativos a una de las representaciones del estado neutro como objetivo
\[
\vdash (\psi_{4}, \psi_{2} )\neq \emptyset
\]
Así, para un estado negativo la función predicción no puede devolver una representación de una secuencia de acciones que lo salve.
\[
\vdash (\psi_{5}, \psi_{1} )= \emptyset
\]
\[
\vdash (\psi_{5}, \psi_{2} )= \emptyset
\]
\[
\vdash (\psi_{5}, \psi_{3} )= \emptyset
\]
\[
\vdash (\psi_{5}, \psi_{4} )= \emptyset
\]

Por lo tanto el resultado del sumatorio es

\[
\sum_{i=1}^{|O|} \frac{|P_{i}|}{|E|}= \frac{3}{5}
\]

Puesto que el conjunto de objetivos positivos sólo tiene un elemento, el resultado  de la otra parte de la fórmula queda
\[
 \frac{\sum^{|E^{\simeq}|}_{j=1} A^{-}_{j} }{|E|} = \frac{1}{5}
\]

\[
\textrm{Estabilidad básica}(\vdash_{\alpha,\beta})=\frac{4}{5}
\]

\end{ejemplo}

\subsubsection{Inestabilidad}
El siguiente parámetro que se va a explicar es la inestabilidad. El concepto de inestabilidad hace referencia a cuánto está de confundido el sistema sobre el las restricciones de existencia del universo. Es decir, imagínese que un sistema tiene entre sus estados objetivos representaciones de estados negativos, o estados positivos de los que la función predicción genera secuencias de actos que hace que se pierdan esos estados. Antes de ver la fórmula para calcular la inestabilidad de una función predicción se definirán algunos conceptos que serán necesarios calcular para obtener la inestabilidad de una función predicción.

\begin{definicion}[Conjunto Objetivos Negativos]
El conjunto objetivos negativos es el conjuntos de elementos que son representaciones de estados del universo del sistema que se estudia. Se representa por $O^{-}$ donde
\[
O^{-} = \{  \psi_{y}: r^{-1}( \psi_{y}) \}
\]
$O^{-} $  es un subconjunto de $O$.
\end{definicion}

\begin{definicion}[Conjunto  Estados Positivos con Alejamiento  hacia el estado-$i$  ]
El conjunto de estados positivos con alejamiento  hacia el estado-$i$  cuyos elementos son estados del ambiente positivos para el sistema desde los cuales el sistema es capaz de predecir una secuencia de acciones para alejarse de él. Su definición es la que sigue
\[
A_{i}^{+}= \{ e_{x}: e_{x}\in \mathbf{E^{+}} \wedge e_{x} = r^{-1}(\psi_{x})\wedge\vdash_{\alpha,\beta}(\psi_{x},\varphi_{i})\neq\emptyset
\}
\]
\end{definicion}

En el concepto de inestabilidad básica se compone de dos conceptos distintos de inestabilidad de una función predicción. El primer concepto, que se corresponderá con el primer término de la fórmula que después se verá, habla de la inestabilidad de objetivos. Así, se hallará en el cálculo los objetivos que son negativos para el sistema y desde cuantos estados del universo la función predicción quiere llegar a ellos.

El segundo concepto, que vendrá reflejado en el segundo término, dará un índice que hace referencia al número de estados neutros que tiene como objetivo la función predicción y para los que la función predicción es capaz de calcular secuencias de acciones partiendo de estados positivos.

La razón de definir la inestabilidad básica como la suma de los dos anteriores conceptos es porque es lógico pensar que una función predicción no debe nunca de provocar comportamientos que lleven a estados negativos o que saque de estados positivos para ir a estados neutros.

\begin{definicion}[Inestabilidad]
Se define inestabilidad de una función predicción de tipo $\alpha,\beta$ como
\[
\textrm{Inestabilidad}_{b}(\vdash_{\alpha,\beta})=\frac{\sum_{i=1}^{|O^{-}|} \frac{|P_{i}|}{|E|}
}{|O^{-}|} + \frac{\sum^{|E^{\simeq}|}_{j=1} A^{+}_{j} }{|E|}
\]

\end{definicion}

\subsubsection{ Estabilidad Total}

El último parámetro que se va a mostrar es el de estabilidad total. La función de la estabilidad total es dar una medida global de lo adecuado que es el valor de la fasa  de un sistema.  Esta medida se define como la diferencia entre la estabilidad básica y la inestabilidad.

\begin{definicion}[Estabilidad Total]
Se define estabilidad total, y se denota por $\textrm{Estabilidad}_{t}$, de una función predicción de tipo $\alpha,\beta$

\[
\textrm{Estabilidad Total} (\vdash_{\alpha,\beta} )=  \textrm{Estabilidadb}
\]
\end{definicion}

Los parámetros que han sido definidos aquí se sirven para medir la función predicción de un sistema. Pero son la base también para medir cómo de buena es una función de aprendizaje. Esto se debe a que si los parámetros definidos sirven para comparar  funciones predicción entre diferentes sistemas; también  para el mismo sistema. Por lo tanto, se puede medir funciones de aprendizaje viendo la diferencia de los parámetros de las funciones predicción que tenga el sistema inicialmente y funciones predicción que genere la función aprendizaje.
Por último hay que notar que los parámetros definidos aquí surgen en función de las restricciones de exoactividad del universo sobre el sistema, pero en una perspectiva evolutiva donde se considere un grupo de sistemas con la misma función predicción cierta inestabilidad pare el sistema podría repercutir con un mayor tiempo en un estado de  exoactividad para el resto de sistemas del grupo.

\section{Discusión}
La TGE acaba de comenzar su camino; pero ya muestra una imagen completamente diferente de la que popularmente se tiene sobre los comportamientos cognitivos. Las ciencias cognitivas  usan el concepto de inteligencia para explicar muchos comportamientos cognitivos. En la TGE se ha decidido eliminar la etiqueta ``inteligencia'' para cualquier concepto. Esa decisión se ha tomado porque el término ``inteligencia'' se usa con un significado distinto en cada ciencia cognitiva, e incluso posiblemente cada uno de los miembros de la comunidad científica de las ciencias cognitivas tenga su propia definición. Así, para evitar polémicas innecesarias y prejuicios sobre los conceptos que contiene la TGE se ha preferido no usar el término  ``inteligencia'' para designar conceptos de la TGE.  Sin duda, la nueva explicación de los comportamientos cognitivos de la TGE basada en una propiedad puede parecer en un primer momento algo difícil de aceptar. Pero que los conceptos de la naturaleza tengan que ser estructurados de una manera diferente a la experiencia sensible del ser humano no es un hecho que sea extraño para el método científico. Como ya se mencionó, la definición de calor usada en física difiere profundamente de la definición popular de calor; por ejemplo, se hace imposible pensar desde la experiencia sensible que un cubito de hielo pueda generar calor porque cuando un ser humano lo toca percibe frío.  En cambio, la física al definir calor de manera absoluta, y relativa a la percepción,  puede  explicar qué ocurre si se pone en contacto un cubito de hielo con nitrógeno en estado líquido. La razón de ese problema reside en que la experiencia sensible de los seres humanos sobre la naturaleza está ligada al rango de percepciones que posee este. Por lo tanto, para llegar a estructurar conocimiento sobre la naturaleza es necesario formular el conocimiento en términos absolutos y no en términos relativos. Así, la formulación de definiciones que hablen en términos absolutos de la naturaleza, y no relativos a la percepción humana, permiten que la ciencia avance, ya que estructuran el conocimiento sobre la naturaleza de manera consistente. Una definición que se realiza en términos relativos a la percepción humana conducirá a la ambigüedad y hará caer en la confusión a la ciencia; de modo que esa definición se deberá de revisar y poner en términos absolutos.

\subsection{Diferencias entre la TGE y Otras Propuestas }
Una pregunta legítima es, ¿qué diferencia la teoría que se está desarrollando de las propuestas para que las leyes de la física incluyan la cognición e inteligencia humana que se mencionaron en la primera sección? La diferencia radica en que la TGE afirma que la noción de computación es una noción fundamental del universo que aparece cuando en la naturaleza surgen sistemas físicos con energía interna y que es la responsable de los exocomportamientos. En cambio, en las propuestas de Penrose y Doyle los exocomportamientos de los seres humanos son una consecuencia de las leyes de la física ya establecidas, y no la manifestación de un aspecto fundamental de la naturaleza.

Una propuesta para dar un trato general a los exocomportamientos es la de Hutter\cite{Hutter2006}, denominada Universal AI$\xi$ Model. Hutter propone un aparato matemático dentro de la teoría de la probabilidad para describir de manera general todos los exocomportamientos. La principal diferencia de la TGE con la de Hutter, es que la propuesta de la Hutter no  es una teoría científica, sino simplemente un aparato matemático para describir exocomportamientos, ya que no plantea ningún hecho sobre la naturaleza que pueda ser falseado mediante experimentos. Por otro lado, no parece el aparato matemático adecuado para describir todos los exocomportamientos. Por ejemplo, para el exocomportamiento sensible que desarrollaría un robot con el programa propuesto para probar el teorema fundamental de los exocomportamientos. No parece razonable describir el exocomportamiento sensible del robot mediante distribuciones de probabilidad porque se sabe que el comportamiento no es aleatorio en ningún caso.

\subsection{ Poniendo a prueba la TGE }
La solidez de una explicación es una característica indispensable para que esta sea aceptada. El método científico reconoce la solidez de un hecho en base a dos criterios fundamentales. La falsabilidad y la reproducibilidad. La falsabilidad quiere decir que existen experimentos para comprobar el hecho. La reproducibilidad es la capacidad de repetir un determinado experimento en cualquier lugar y por cualquier persona sin que sus resultados contradigan el hecho que se pone a prueba en el experimento. La razón de estos dos criterios es que en ciencia es imposible probar que un hecho se va a cumplir siempre. Pero si se realizan suficientes experimentos sobre el hecho y ninguna lo consigue falsear, usando inducción se considera que siempre se cumplirá. En una teoría científica los hechos que deben cumplir los anteriores criterios son los postulados. Por lo tanto, no se puede obviar estas cuestiones para que la TGE tenga una consideración de teoría científica. Así, la primera cuestión que se debe de estudiar es si es posible realizar experimentos cuyos resultados validen los postulados, ya que si se realiza un número importante de experimentos y ninguno sale negativo entonces podrá confiarse en las explicaciones que aporta la TGE sobre la naturaleza.

\subsection{Falseabilidad del Postulado Primero}

El postulado primero dicta que sólo existen tres tipos de carácter elemental para la fasa de un sistema. La manera de conocer si este postulado es falso sería estudiar sistemas biológicos y no biológicos de la naturaleza y encontrar uno que genere un comportamiento donde el carácter de su fasa no sea ni un carácter elemental, ni un híbrido de ellos. Por lo tanto, el primer postulado cumple el criterio de falseabilidad.

\subsection{Falseabilidad del Postulado Segundo}

A continuación, se va a estudiar la falseabilidad del postulado segundo mediante su forma lógica. Este análisis se hace con una doble intención. Primero para buscar el tipo de experimentos que ponen a prueba  el segundo postulado. La segunda intención es debido a que normalmente cuando el lector se enfrenta a aceptar el postulado segundo le surgen cuestiones relacionadas con los posibles casos de entornos de la naturaleza donde se aplicaría el postulado. Existen ocho situaciones que surgen al confrontar tipos de sistemas con los tipos de universos sobre las que el postulado se puede aplicar:
\begin{enumerate}
  \item	Sistemas que no tienen capacidad de moverse

  \begin{description}
  \item[a)]	En un universo con sensibilidad al acto
  \item[b)]	En un universo sin sensibilidad al acto

\end{description}
  \item	Sistemas con capacidad de moverse en universos sin sensibilidad al acto

  \begin{description}
          \item[a)]	Sólo hay sistemas sin representación
          \item[b)]	Sólo hay sistemas con representación
          \item[c)]	Sistemas con y sin representación
        \end{description}
  \item	Organismos con capacidad de moverse en universos sin sensibilidad al acto
    \begin{description}
          \item[a)]	Sólo hay sistemas sin representación
          \item[b)]	Sólo hay sistemas con representación
          \item[c)]	Sistemas con y sin representación
        \end{description}
\end{enumerate}
El análisis consistirá en evaluar la forma lógica del postulado mediante su tabla de verdad en cada una de las ocho situaciones. Para ellos se empezará expresando la forma lógica del postulado. El postulado se corresponde con la siguiente fórmula lógica:
\[
(s \wedge  r \rightarrow r')  \wedge  (s \wedge  n \rightarrow  n')
\]
\begin{itemize}
  \item $s$ = `` universo con sensibilidad al acto''
  \item $r$ = ``hay sistemas que se mueven en función de una  representación''
  \item $n'$ = ``hay sistemas que no se mueven en función de una  representación''
  \item $r$ = ``persisten sistemas que se mueven en función de una  representación''
  \item $n'$ = ``persisten sistemas que no se mueven en función de una  representación''
  \item $\wedge$ es el operador lógico ``y''
  \item $\rightarrow$ es el operador lógico implicación
  \item V es el valor de VERDAD
  \item F el de FALSO
\end{itemize}

 A continuación, se va  a estudiar la fórmula lógica del postulado segundo en las distintas situaciones. Se comenzará el estudio de un conjunto de situaciones exponiéndose un caso de la naturaleza que está ligado a ellas. Seguidamente se llevará a cabo un análisis en las tablas del valor de verdad del postulado en las situaciones en la que está incluido el caso.

\begin{description}
  \item[I] SISTEMAS SIN CAPACIDAD DE MOVERSE
  \newline
  \newline
  Ejemplo:
\textit{Hay determinadas plantas subacuáticas que poseen una especie de flotadores en sus hojas lo que posibilita que la planta se sitúe y permanezca en una posición vertical y no sobre el fondo marino. De esa manera, las hojas de esas plantas acuáticas están más cerca de la superficie, lo que permite que reciban más energía solar, ya que la diferencia de energía recibida en las hojas varía dependiendo de la profundidad a la que se encuentren situadas.
}

 Las plantas acuáticas del ejemplo no usan una representación del entorno para desarrollar la citada estrategia evolutiva, la cual es beneficiosa para su existencia ¿son entonces un contraejemplo del postulado segundo?

El caso anterior está incluido en las situaciones en las que el sistema no puede moverse. La tabla de verdad de  estas situaciones es la siguiente:

\begin{tabular}{| c| c | c| c |c | c| c |}
\hline
 CASO & $s$ & $r$ & $n$ & $s \wedge r$ & $s\wedge n$ &  $(s \wedge  r \rightarrow r')  \wedge  (s \wedge  n \rightarrow  n')$  \\
\hline
a)& V &	F &	F &	F &	F &	V \\
\hline
b)& F &	F &	F & F &	F &	V
\\
\hline
\end{tabular}

En la tabla de verdad se observa que tanto en el caso a) como en el b) la fórmula es verdadera. Esto es debido a que si los sistemas que se estudian no tienen capacidad de moverse entonces los antecedentes de los condicionales son falsos, y en un condicional cuando su antecedente es falso independientemente del valor del consecuente, el condicional es verdadero. Por lo tanto en este caso el postulado es verdadero de manera trivial, ya que no afecta de ninguna manera a los sistemas del universo. Por lo tanto en estas situaciones pueden seleccionarse sistemas que no se muevan, y la selección se hace en función de la estrategia evolutiva  que presente el sistema.
Así, la respuesta a si el ejemplo presentado es un contraejemplo al postulado la respuesta es no. Esas plantas no tienen capacidad de moverse pero su estrategia evolutiva es beneficiosa al permitir que consiga más energía solar, por lo que son seleccionadas por la naturaleza sin necesidad de tener una representación del entorno.

  \item[II] SISTEMAS CON CAPACIDAD DE MOVERSE EN UNIVERSOS CON SENSIBILIDAD AL ACTO
    \newline
    \newline
      Ejemplo:
\textit{Un ejemplo de sistemas que se mueven  y pueden sobrevivir sin una representación en universos son los elasmobranquios (tiburones) cuando se encuentran en su fase de embrión en el huevo. Para que se distribuya el oxígeno a todos los tejidos del cuerpo, los fluidos del huevo (ooplasma) deben moverse continuamente. Así, los elasmobranquios  necesitan moverse para obtener oxígeno, de lo contrario mueren. Cuando están en los huevos en su fase embrión desarrollan este movimiento sin un sistema nervioso (y por ende sin una representación del entorno) simplemente tienen un patrón  de contracción  muscular que los hace moverse.}

Los seres vivos del ejemplo no tienen una representación del entorno pero se mueven y no son pasan a un estado de exoinactivos  ¿son entonces un contraejemplo del postulado segundo?

El ejemplo anterior está incluido en el conjunto de situaciones en las que el sistema puede moverse y el universo no es sensible a las condiciones iniciales. Esto es debido a que el embrión no puede salir del huevo y que no  importa la posición o la dirección en la que se mueva el elasmobranquio, el sólo hecho de que él se mueva le permite obtener el oxígeno. Así, en la práctica, el huevo es un universo sin sensibilidad al acto para el elasmobranquio. La tabla de verdad de  estas situaciones es la siguiente:

\begin{tabular}{| c| c | c| c |c | c| c |}
\hline
 CASO & $s$ & $r$ & $n$ & $s \wedge r$ & $s\wedge n$ &  $(s \wedge  r \rightarrow r')  \wedge  (s \wedge  n \rightarrow  n')$  \\
\hline
a)& F &	F &	V &	F &	F &	V \\
\hline
b)& F &	V &	F &	F &	F &	V \\
\hline
c)& F &	V &	V &	F &	F &	V \\
\hline
\end{tabular}

Como se ve en la tabla de verdad de este conjunto de situaciones, en ninguna de ellas pueden volver falsa la fórmula lógica que contiene el postulado segundo, ya que al ser los antecedentes falsos de nuevo por las características del condicional este es verdadero. Por lo que  la respuesta para la pregunta, sobre si el ejemplo presentado es un contraejemplo que se planteo, vuelve a ser que no.  En entornos sin sensibilidad al acto la evolución puede seleccionar sistemas que se muevan y no usen una representación del entorno para realizar su comportamiento y la fórmula lógica que contiene el segundo postulado sigue siendo verdad.

\item[III] SISTEMAS CON CAPACIDAD DE MOVERSE EN UNIVERSOS CON SENSIBILIDAD AL ACTO
    \newline
    \newline
\begin{tabular}{| c| c | c| c |c | c| c |}
\hline
 CASO & $s$ & $r$ & $n$ & $s \wedge r$ & $s\wedge n$ &  $(s \wedge  r \rightarrow r')  \wedge  (s \wedge  n \rightarrow  n')$  \\
\hline
a)& V &	F &	V &	F &	V &	?\\
\hline
b)& V &	V &	F &	V &	F &	?\\
\hline
c)& V &	V &	V &	V &	V &	?\\
\hline
\end{tabular}

 Como muestra la tabla de verdad, en estos tres casos el valor de verdad del condicional no se puede saber por el antecedente, es necesario conocer el valor de verdad del consecuente para saber qué valor contiene la fórmula lógica del postulado. Por lo tanto, para que los experimentos pongan a prueba el postulado segundo deben coincidir con los valores de los antecedentes de alguno de los casos de esta tabla de verdad. Si el postulado se cumple en la naturaleza el consecuente del condicional debe ser verdad en cada experimento. Una cuestión que debe ser tenida en cuenta al planificar los experimentos es que la estabilidad  básica debe ser alta y la inestabilidad baja, ya que en el análisis lógico no se ha tenido en cuenta explícitamente lo correcta que es la representación del entorno, sino que implícitamente se ha supuesto que las representaciones lo eran.
\end{description}

Con cada experimento que se lleve a cabo donde el consecuente resulte verdadero el postulado segundo cobrará solidez, y a su vez, la TGE recibirá legitimidad como teoría para explicar los comportamientos en general y los sensibles en particular.Los tunicados, de los que ya se ha hablado antes, son una prueba biológica de que la evolución selecciona sistemas con mecanismos de representación del ambiente cuando el sistema se mueve en un ambiente con sensibilidad al acto, ya que de esa manera se disminuyen drásticamente las posibilidades de perecer. Otra prueba del postulado son los elasmobranquios citados anteriormente, ya que la selección natural ha provocado que para cuando dejan el huevo y deben moverse en el mar posean un cerebro que les permita tener una representación. En este caso el mar es un ambiente con sensibilidad al acto y la selección natural ha provocado que tanto los tunicados como los elasmobranquios usen un cerebro para albergar una representación del ambiente y moverse. Pero no son suficientes dos pruebas para que sea sólido el postulado segundo, es necesario realizar experimentos y verificar que sus resultados son consistentes con los del postulado. La manera más factible de llevar a cabo experimentos que pongan a prueba el segundo postulado es mediante experimentos informáticos. Recuérdese  que la definición de sistema dice que un sistema es definido tanto por una relación física como lógica, por lo que los experimentos informáticos son tan importantes como los que se pudieran hacer con sistemas definidos por relaciones físicas.  Por otro lado, los experimentos informáticos permiten tener totalmente controlado un experimento y se pueden obtener resultados en una cantidad de tiempo asequible.

Un experimento que se está realizando es usar un mundo virtual con sensibilidad al acto y restricciones de existencia, en el que se recogen los tiempos de persistencia de los agentes que llevan a cabo un comportamiento sensible, aleatorio y posicional . Una vez recogidas todas las mediciones de tiempos se debe confrontar los tiempos de los agentes con comportamientos aleatorios y posicionales  con los de los comportamientos sensibles. Según el segundo postulado los resultados de los experimentos deben de mostrar una cantidad de tiempo de existencia mucho más alta en los comportamientos sensibles que en los otros dos tipos de comportamientos elementales. Los primeros resultados parecen confirmar el segundo postulado.

\subsection{TES y Percepción}

Otra cuestión que debe de comentarse es la relación de  la TCCV con la percepción. La TCCV no habla de una manera explícita de la percepción, porque al igual que la TCCV no especifica que sea necesario un concreto mecanismo de computación tampoco dice que sea necesario un mecanismo concreto de percepción. Pero sin duda, existe una cuestión que hace completamente necesaria la existencia de mecanismos de percepción. Un sistema no puede tener una representación del entorno si carece de mecanismos para percibir el entorno, ya que es necesario que se pueda establecer una función entre los estados del universo y los estados del subsistema que usa el sistema para crear la representación. Así, aunque el nivel de abstracción del marco matemático convierte la cuestión de la percepción en implícita, las relaciones que existen entre  percepción y  la TES pueden ser muy diferentes de algo trivial. El número de fórmulas que tenga el lenguaje que usa el sistema para describir los estados del universo y el poder expresivo de este lenguaje deben de estar en consonancia con los mecanismos de percepción que use el sistema. De modo que los mecanismos de percepción permitan usar todo el poder expresivo del lenguaje para representar el ambiente.

\subsection{TGE y la Teoría de la Evolución}
Una cuestión importante que sobre la TGE es si es compatible con la teoría de la evolución. La respuesta es afirmativa. Para ver la relación entre la teoría de la evolución y la TGE, primero se debe mencionar que se puede definir una relación de orden parcial sobre las distintas arquitecturas funcionales no redundantes usando el número de unidades funcionales que poseen las arquitecturas funcionales. Esa relación tiene un mínimo que es la arquitectura de tipo I. Si se dibuja la relación que existe usando la complejidad de los tipos de ecuaciones aparece la figura \ref{fig:evo} La idea que se propone en la TGE es que por el postulado segundo esta estructura es un patrón que sigue la naturaleza, ya que la selección natural obliga a que se desarrollen sistemas que realicen comportamientos sensibles que permitan un mejor cumplimiento de las restricciones de existencia de un universo. Es decir, respecto a los sistemas que desarrollan comportamientos sensibles la TGE propone que en la naturaleza primero aparecieron sistemas de tipo I y a partir de este surgieron el resto.

Pero subir de nivel de complejidad no es la única manera en la que la evolución se relaciona con la TGE. Por el postulado segundo la selección natural también puede verse como un procedimiento que realiza una exploración de las soluciones de la ecuación, de manera que persisten las soluciones con mayor índice de estabilidad básica y menor índice de inestabilidad de a la ecuación que van surgiendo.

\begin{center}
\begin{figure}
  \centering
    \includegraphics{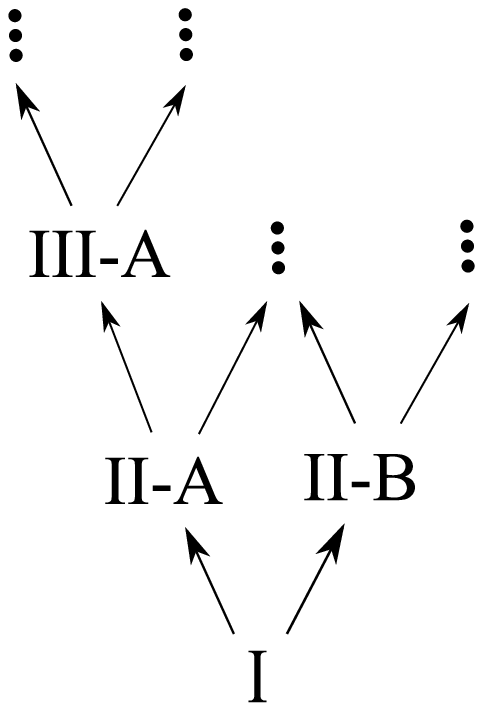}
  \caption{Relación de orden definida por el aumento de unidades funcionales en las las diferentes arquitecturas. Como se puede ver en la relación existe un elemento mínimo, la arquitectura I.}
  \label{fig:evo}
\end{figure}
\end{center}

\subsection{La TGE y los Comportamientos Distribuidos}

Hasta ahora se ha hablado de exocomportamientos cuando estos se han realizado por un sistema, pero ¿qué hay sobre los comportamientos distribuidos, que son realizados por muchos sistemas? En realidad, no cambia apenas nada. Todos los sistemas que realizan un comportamiento distribuido son, a su vez, un sistema. Todas las hormigas que realizan un comportamiento distribuido forman un hormiguero, todas las abejas que realizan un comportamiento distribuido forman una colmena,... . Es decir, a su vez todos los sistemas que realizan un comportamiento distribuido pueden verse como subsistemas de un sistema y el comportamiento distribuido sigue siendo un exocomportamiento, la diferencia está en que la relación que hay entre ellos no es física sino lógica. Así, los exocomportamientos distribuidos también serían explicados por la TGE  aunque  evidentemente son necesarias ciertas modificaciones en el aparato matemático para describirlos con ecuaciones en el marco matemático que usa la TGE.

\subsection{Presente y Futuro de la TGE}

Las leyes de la mecánica que rigen las partículas elementales de la naturaleza dirigen la evolución del estado de la naturaleza; pero intentar entender o predecir el estado futuro de la naturaleza a nivel macroscópico es un problema inabordable mediante las leyes del mundo microscópico debido al orden de magnitud del número de partículas de los sistemas macroscópicos ($10^{25}$ partículas). Por lo tanto, el usar teorías que explican el nivel microscópico para explicar los comportamientos a nivel macroscópico es un enfoque impracticable, ya que requeriría la resolución de un número increíblemente grande de ecuaciones diferenciales; pero no sólo eso, sino que introducir las condiciones iniciales de tal sistema sería imposible. La ciencia dispone en determinados aspectos de la naturaleza de una teoría macroscópica que permite explicar la naturaleza a nivel microscópico y donde además el nivel macroscópico puede ser explicado desde el nivel microscópico. Por ejemplo, la mecánica o la termodinámica. Pero en el caso de los comportamientos que son fenómenos microscópicos el lenguaje de las teorías de la psicología, la etología  o la I.A. no puede ser enlazado con el nivel microscópico.  El resultado obtenido durante la investigación que se ha presentado en este artículo es una teoría para explicar los comportamientos  que sirve de puente entre el nivel macroscópico y el nivel microscópico de la naturaleza. La TGE es una propuesta que permite enlazar los comportamientos de los sistemas a nivel macroscópico, descritos por ejemplo por la etología, la psicología o la I.A. con el nivel microscópico de las leyes que rigen las interacciones entre las partículas fundamentales.  Eso es posible gracias a que por un lado la teoría parte del concepto de sistema, cuya formación y existencia se pueden explicar perfectamente desde el nivel microscópico; y por otro lado la teoría explica las alteraciones que produce el sistema a nivel macroscópico a cargo de su energía interna.

Este documento establece los pilares y la estructura de la TGE, pero queda pendiente construir el resto del edificio. En la TGE actualmente existe una mezcla de preguntas en investigación básica sobre la propia TGE y otro conjunto de preguntas de investigación aplicada.  Preguntas en investigación básica como ¿Cuantas arquitecturas funcionales existen? ¿Qué soluciones podemos encontrar a las ecuaciones de las diferentes arquitecturas funcionales? ¿Hay más parámetros para medir las unidades funcionales a parte de los de estabilidad e inestabilidad? ¿Existen exocomportamientos en el conjunto de exocomportamientos sensibles no ordenados que pueden permitir que un sistema persista en un universo con sensibilidad al acto?

Y en investigación aplicada hay preguntas como: ¿Qué exocomportamientos sensibles de la naturaleza explican exactamente  cada una de las diferentes ecuaciones y cada una de las diferentes soluciones? ¿Y los parámetros de estabilidad e inestabilidad?  ¿Se puede aplicar esta teoría a la psicología? ¿La podemos aplicar a I.A. distribuida? ¿Podemos usarla para entender mejor las redes neuronales?

Como se ha mencionado al principio de esta sección, los conceptos que aporta la TGE para explicar los comportamientos puede que no sean fáciles de admitir. Pero si los experimentos respaldan la teoría, un interesante horizonte se abre ante nuestros ojos a través de la teoría general del exocomportamiento.

\section*{Agradecimientos}

Me gustaría agradecer a  todas las personas que me han ayudado en este largo pero emocionante viaje que ha sido el llegar a desarrollar la teoría ge\-neral de los exocomportamientos. Entre todos ellos quiero destacar a Rodolfo Llinás al que agradezco profundamente nuestras inolvidables conversaciones durante estos dos últimos años, porque han impulsando enormemente mi deseo de hallar respuestas sobre la naturaleza, y al que sólo puedo considerar un auténtico maestro para mí como persona y científico. A continuación, me gustaría dar las gracias a Enrique Alonso por alentarme a trabajar en los fundamentos y el armazón de la TGE haciéndome ver donde había saltos que debían ser completados. A María Teresa López Bonal y Arnau Ramisa porque con sus ojos han logrado que este documento tenga muchos menos errores de los que tenía. A Luis de la Ossa porque ha escuchado demasiadas veces mis meditaciones sobre la TGE. A las interesantes críticas de Enric Trillas, a los comentarios de José Mira Mira. Agradezco profundamente a mis padres su apoyo durante tantos años. A José María Cabañes quiero agradecer  su interés en que sus alumnos intenten pensar por ellos mismos y  por haber tenido aquella capacidad de ver que aquella primera idea, que le presenté, era una semilla de donde podía germinar un gran árbol. Quiero dar las gracias a Juan Ángel Aledo por enseñarme especialmente que las matemáticas eran un lenguaje para transmitir ideas. A María Gracia Manzano le agradezco que sea mi guía por ese complejo mundo de la lógica matemática  y por  compartir sus enormes conocimientos conmigo. A Miguel Ángel Graciani el que me haya escuchado tantas, y tantas, veces  y hacerme críticas constructivas sobre qué cuestiones debía afianzar en mi trabajo. Agradezco a todos los que me han prestado de su tiempo, hayan estado mucho,  o poco, de acuerdo conmigo; porque que alguien te regale parte de su tiempo es el regalo más precioso que se puede hacer.

\bibliographystyle{elsarticle-num}

\bibliography {Bibliografia}
\end{document}